\documentclass[12pt]{amsart}

\usepackage[margin=1in]{geometry}
\usepackage[T1]{fontenc}
\usepackage[utf8]{inputenc}
\usepackage[numbers,square,sort&compress]{natbib}

\usepackage{amsmath,amsfonts,bm}

\def\eqref#1{equation~\ref{#1}}
\def\Eqref#1{Equation~\ref{#1}}

\def\1{\bm{1}}

\DeclareMathAlphabet{\mathsfit}{\encodingdefault}{\sfdefault}{m}{sl}
\SetMathAlphabet{\mathsfit}{bold}{\encodingdefault}{\sfdefault}{bx}{n}

\def\sH{{\mathbb{H}}}

\newcommand{\R}{\mathbb{R}}

\DeclareMathOperator*{\argmin}{arg\,min}

\usepackage[ruled]{algorithm2e}
\usepackage[colorlinks=true,linkcolor=blue,citecolor=blue,urlcolor=blue]{hyperref}
\usepackage{url}
\usepackage[american]{babel}
\usepackage{subfiles}
\usepackage{algorithm2e}
\usepackage{graphicx}
\usepackage{mathtools}
\usepackage{caption}
\usepackage{amsmath}
\usepackage{bbm}
\usepackage{amsthm}
\usepackage{booktabs}

\newtheorem{theorem}{Theorem}

\newtheorem{remark}[theorem]{Remark}
\newtheorem{lemma}[theorem]{Lemma}
\newtheorem{proposition}[theorem]{Proposition}
\usepackage{enumitem}
\usepackage{cleveref}

\newlist{blist}{enumerate}{1}
\setlist[blist]{label=(B\arabic*), ref=B\arabic*, leftmargin=2.6em, itemsep=3pt}

\crefname{blisti}{Assumption}{Assumptions}
\Crefname{blisti}{Assumption}{Assumptions}
\crefname{assumption}{Assumption}{Assumptions}
\Crefname{assumption}{Assumption}{Assumptions}

\crefname{lemma}{Lemma}{Lemmas}
\Crefname{lemma}{Lemma}{Lemmas}

\usepackage[acronym]{glossaries}
\makeglossaries
\newacronym{qp}{QP}{Quadratic Program}
\newacronym{gw}{GW}{Gromov--Wasserstein}
\newacronym{egw}{EGW}{Entropic Gromov--Wasserstein}
\newacronym{ot}{OT}{Optimal Transport}
\newacronym{eot}{EOT}{Entropic Optimal Transport}
\newacronym{amd}{AMD}{Averaged Mirror Descent}
\newacronym{md}{MD}{Mirror Descent}
\newacronym{kl}{KL}{Kullback-Leibler}

\newcommand{\rA}{\mathrm{A}}
\newcommand{\rB}{\mathrm{B}}
\newcommand{\rC}{\mathrm{C}}

\newcommand{\rM}{\mathrm{M}}

\newcommand{\rP}{\mathrm{P}}

\newcommand{\intercal}{\mathsf{T}}
\newcommand{\maxnorm}[1]{\left\|#1\right\|_{\max}}

\newcommand{\cK}{\mathcal{K}}

\renewcommand{\vec}{\mathrm{vec}}
\renewcommand{\sH}{\mathsf{H}}

\title[Averaged Mirror Descent and Dual Gradient Methods for EGW]{Averaged Mirror Descent and Dual Gradient Methods: Convergent Algorithms for Entropic Gromov--Wasserstein Problems}

\author[Joanna Marks]{Joanna Marks$^{\star}$}
\address{Department of Mathematics, Imperial College London}
\email{joanna.marks23@imperial.ac.uk}

\author[Gabriel Rioux]{Gabriel Rioux$^{\star}$}
\address{Department of Mathematics, Imperial College London}
\email{g.rioux@imperial.ac.uk}

\author[Riccardo Passeggeri]{Riccardo Passeggeri}
\address{Department of Mathematics, Imperial College London}
\email{riccardo.passeggeri@imperial.ac.uk}
\thanks{
J. Marks is supported by EPSRC through the StatML CDT programme, grant no. EP/Y034813/1. R. Passeggeri is partially supported by EPSRC grant UKRI2396.
}

\thanks{$^{\star}$Denotes equal contribution to this work.}

\date{}

\begin{document}

\begin{abstract}
The Gromov--Wasserstein (GW) distance measures the discrepancy between metric measure (mm) spaces and identifies optimal alignments between them based solely on their intrinsic structure. Since it identifies isomorphic mm spaces, it provides a natural notion of distance for heterogeneous datasets which may admit isomorphic representations. In order to accelerate computation of GW distances, many practitioners employ entropic regularization to obtain an Entropic GW (EGW) problem. The most popular EGW solver is the Mirror Descent (MD) algorithm, which reduces EGW computations to an iterative process where an entropic optimal transport (EOT) problem is solved at each iteration. Despite its widespread use, the convergence of MD for this problem has only been established for restricted classes of costs. On the other hand, a recently proposed dual gradient method is available for general costs, but requires a choice of step size which depends on the regularization parameter. To address these two issues, we introduce Averaged Mirror Descent (AMD), which averages consecutive MD steps, and prove its convergence for arbitrary costs. Then, we establish that the dual gradient method with a fixed step size also converges for arbitrary costs at the cost of a more complicated iteration. 
In both cases, we also account for inexact iterations which are inescapable in practice.  We compare the empirical performance of these methods across various settings and, in particular, show that AMD and the dual gradient method both converge on an example where classical MD fails.

\end{abstract}

\maketitle

\section{Introduction}
\gls*{ot} provides a principled framework for comparing and transforming probability distributions according to the geometry of the underlying metric space \citep{villani2008optimal}. Computing \gls*{ot} between finitely discrete distributions supported on $N$ points requires solving a linear program in $N^2$ variables which can be done in $O(N^3\log(N))$ operations via the network simplex algorithm \citep{peyre2019computational}. This undesirable cubic scaling can be circumvented by regularizing \gls*{ot} using an entropic penalty yielding the \gls*{eot} problem which can be solved in $O(N^2\log(N))$ operations via Sinkhorn's diagonal scaling algorithm \citep{cuturi2013lightspeed}. This has since driven considerable interest in the ML community, with applications spanning generative modeling \citep{arjovsky2017, gulrajani2017improved, genevay2018learning, bortoli2021diffusion, lipman2023flow}, sampling \citep{vargas2023denoising, liu2026adjoint} and image recognition \citep{li2013novel,rubner2000earth}, among others. 

An important  consideration when using \gls*{ot} in applications is that it requires specifying a cost of transporting points between the supports of the underlying distributions which captures a desired correspondence. In cases where the distributions represent datasets of different types, the cost is generally chosen in an \emph{ad hoc} manner and may fail to account for intra-domain relationships between the underlying objects. This is particularly important when comparing objects which may have many natural representations within the same space such as in aligning shapes related by an unknown rotation or matching graphs across different node representations. The \gls*{gw} distance, introduced in \cite{memoli2011gromov}, enables comparing metric measure spaces by comparing their intra-domain distortions.  Crucially, the \gls*{gw} distance vanishes if and only if the two metric measure spaces are isomorphic and provides a  coupling $\pi^\star$ which can be thought of as an optimal fuzzy alignment of the spaces. These properties have made \gls*{gw} distances a versatile tool across many domains including language modeling \citep{alvarez-melis-jaakkola-2018-gromov, kawakita2024gromov}, generative modeling \citep{bunne2019learning, klein2024genot}, graph matching \citep{brogat2022learning,xu2019gromov,rioux2024limit}, and single-cell genomics \citep{demetci2022scot, klein2024genot}.

Despite these appealing properties, computing the \gls*{gw} distance generally requires solving a non-convex \gls*{qp} which is computationally hard. Nevertheless, fast iterative methods for approximately solving regularized \gls*{gw} problems have been proposed in \cite{peyre2016gromov, solomon2016entropic} and appear to work well in practice. These methods regularize the \gls*{gw} problem using the entropic penalty and linearize the objective so that each iteration only requires solving an \gls*{eot} problem; this  is the core of the \gls*{md} algorithm.  Although this method is widely used in practice, the convergence analysis of this method is limited and the impact of inexactness in Sinkhorn's method for solving the \gls*{eot} problem is unaccounted for in the literature. Related approaches to accelerating \gls*{gw} computations include convex/conic relaxations \citep{sejourne2021unbalanced,vincent2021semi}, semi-definite relaxations \citep{chen2023semidefinite}, and slicing-based methods  \citep{vayer2019sliced,gong2026sliced}, but \gls*{egw} remains the most popular method for practical applications.

Another line of work, \citet{rioux2024entropic}, introduced provably convergent algorithms for the \gls*{egw} problem between Euclidean distributions by leveraging a variational formulation introduced in \cite{zhang2024gromov}. Subsequently, \cite{houry2026} extended this dual formulation to costs of conditionally  negative type (CNT) and showed that standard gradient methods with a fixed step size (independent of the regularization parameter) converge and coincide with \gls*{md} for CNT costs. Moreover, their work shows, via a numerical example, that \gls*{md} iterations can oscillate and fail to converge if the cost is not CNT.  Most recently, \cite{rioux2026discrete} established a variational form of \gls*{egw} which covers arbitrary finitely discrete metric measure spaces and developed convergent algorithms with a stepsize that decays to $0$ as the regularization parameter is decreased; the case of CNT costs effectively corresponds to conditionally negative semidefinite \gls*{gw} cost matrices in that work. Hence, two natural questions arise: (1) can a simple modification of  \gls*{md} be shown to converge for arbitrary costs and (2) can a dual gradient method for arbitrary costs be shown to converge with a stepsize that does not depend on the regularization parameter.     %

The present work answers both of these questions positively and additionally shows that the error from inexact resolution of the \gls*{eot} problems can be accounted for in the analysis of these algorithms.    
To address (1) we introduce a variant of the mirror descent algorithm which takes less aggressive steps and is provably convergent. We dub this method \gls*{amd}, as we average the standard MD step with the previous iterate.
\gls*{amd} can be shown to coincide with a conditional gradient method, enabling us to obtain non-asymptotic convergence bounds. In addition, we go beyond the classical analysis of conditional gradient methods by establishing its convergence when subproblems are solved inexactly. %
As for (2), we leverage the dual formulation of \gls*{egw} from \cite{rioux2026discrete} and the proof technique from \cite{houry2026} to show that a standard gradient method with fixed stepsize (independent of the regularization parameter) converges for arbitrary costs. Crucially, the iterates for the dual gradient method do not coincide with the \gls*{md} updates for indefinite \gls*{gw} cost matrices which indicates that the failure of \gls*{md} to converge on some numerical examples is due to its coarse approximation of the true gradient step.

\section{Optimal transport and Gromov--Wasserstein problems}
\paragraph{Optimal transport.}
Fix two finite sets $\mathcal X_0$ and $\mathcal X_1$, and let $\mu_0,\mu_1$ be probability distributions on $\mathcal X_0,\mathcal X_1$ respectively (denoted as $\mu_0\in\mathcal P(\mathcal X_0)$, $\mu_1\in\mathcal P(\mathcal X_1)$).
In addition, let $c:\mathcal X_0\times \mathcal X_1\to \mathbb R$ be a cost function. The \gls*{ot} problem reads 
$\mathsf{OT}_c(\mu_0,\mu_1)\coloneqq \inf_{\pi\in\Pi(\mu_0,\mu_1)}\int cd\pi,
$
where $\Pi(\mu_0,\mu_1)$ denotes the collection of all couplings of $(\mu_0,\mu_1)$, i.e., distributions on $\mathcal X_0\times \mathcal X_1$ whose $\mathcal X_i$-marginal is $\mu_i$ for $i\in\{0,1\}$. The \gls*{eot} problem regularizes the \gls*{ot} problem as in \cite{cuturi2013lightspeed}
\begin{equation}
   \begin{aligned} 
\label{eq:EOTPrimal_intro}
\mathsf{EOT}_c^{\varepsilon}(\mu_0,\mu_1)\coloneqq    \inf_{\pi\in\Pi(\mu_0,\mu_1)} \int cd \pi+\varepsilon{\mathrm{KL}}(\pi\|\mu_0\otimes \mu_1).
    \end{aligned}
\end{equation} 
where $\varepsilon>0$ is a regularization parameter and ${\mathrm{KL}}(\pi\|\mu_0\otimes \mu_1)$ is the \gls*{kl} divergence between $\pi$ and the product measure $\mu_0\otimes \mu_1$ defined by $\int \log(d\pi/d(\mu_0\otimes \mu_1))d\pi$ when $\pi\ll\mu_0\otimes \mu_1$ and takes the value $+\infty$ otherwise.\footnote{The notation $\pi\ll\mu_0\otimes \mu_1$ indicates that $\pi$ is absolutely continuous with respect to $\mu_0\otimes \mu_1$, that is, $\mu_0\otimes \mu_1(A)=0$ implies that $\pi(A)=0$ for any measurable set $A$.} The added penalty makes the resulting problem strictly convex and enforces a certain structure on solutions of \Eqref{eq:EOTPrimal_intro}; this structure is at the heart of \cite{cuturi2013lightspeed} which proposed to solve the \gls*{eot} problem via Sinkhorn iterations, see \Cref{sec:Sinkhorn} for details.  

\paragraph{Gromov-Wasserstein Problems.}
Given finite sets $\mathcal X_0,\mathcal X_1, \mu_0\in\mathcal P(\mathcal X_0),\mu_1\in\mathcal P(\mathcal X_1)$, and kernels $\kappa_0:\mathcal X_0\times \mathcal X_0\to \mathbb R$ and $\kappa_1:\mathcal X_1\times \mathcal X_1\to \mathbb R$, the $p$-\gls*{gw} problem ($p>0$) is given by \cite{memoli2011gromov}
\[
    \mathsf{GW}_{p}(\mu_0,\mu_1)\coloneqq \inf_{\pi\in\Pi(\mu_0,\mu_1)}\iint \left|\kappa_0(x,x')-\kappa_1(y,y')\right|^pd\pi\otimes \pi(x,y,x',y').
\]
As with \gls*{eot} it has been proposed in \cite{peyre2016gromov,solomon2016entropic} to penalize \gls*{gw} using the \gls*{kl} divergence, viz. 
\[
    \mathsf{EGW}^{\varepsilon}_{p}(\mu_0,\mu_1)\coloneqq \inf_{\pi\in\Pi(\mu_0,\mu_1)}\iint \left|\kappa_0(x,x')-\kappa_1(y,y')\right|^pd\pi\otimes \pi(x,y,x',y')+\varepsilon\mathrm{KL}(\pi\|\mu_0\otimes \mu_1),
\]
where $p>0$ and $\varepsilon>0$.
The resulting problem is called the \gls*{egw} problem and is generally non-convex unless $\varepsilon>0$ is sufficiently large. However, in applications $\varepsilon>0$ is taken to be small in order to approximate solutions of \gls*{gw}. Concretely, \cite{karumanchi2025approximation} established that solutions of the regularized problem converge to those of the original problem exponentially fast (in $\varepsilon$) if the quadratic form above is concave, but that this convergence is as slow as linear otherwise. 

Given that $\mathcal X_0=\{x_0^{(i)}\}_{i=1}^{N_0},\mathcal X_1=\{x_1^{(j)}\}_{j=1}^{N_1}$ are finite sets, 
solving \gls*{egw} is equivalent to solving a regularized \gls*{qp} in $k=N_0N_1$ variables which we write as
\begin{equation}\label{eq:EGWQP}
\mathsf{EGW}_{p}^{\varepsilon}(\mu_0,\mu_1) = \inf_{\substack{\rA x = b\\ x\geq 0}} \left\{\frac{1}{2}x^{\intercal}\rC x-\varepsilon\mathsf{H}(x)+\varepsilon \mathsf H_{\mu_0} +\varepsilon \mathsf H_{\mu_1} \right\},
\end{equation}
where $\{x\in \mathbb R^k: \rA x = b, x\geq 0\}$ accounts for the coupling constraints, $\rC$ is the cost matrix with entries $\rC_{i,j} = 2\left|\kappa_0(x_{0}^{(i_0)},x_{0}^{(j_0)})-\kappa_1(x_1^{({i_1})},x_1^{({j_1}})\right|^p$, where $(i_0, i_1)$ and $(j_0, j_1)$ are pairs such that $i = (i_1-1)N_0 + i_0$ and $j = (j_1-1)N_0 + j_0$, 
$
\mathsf H(x)=-\sum_{i=1}^{k} x_i\log(x_i)
$
is Shannon's entropy which we treat as a function on the non-negative orthant, and $\mathsf H_{\mu_0}= -\sum_{x\in\mathcal X_0}\mu_0(\{x\})\log(\mu_0(\{x\}))$ and similarly for $\mathsf H_{\mu_1}$. Throughout, we opt to work with the symmetric matrix $\rM=\frac{1}{2}(\rC+\rC^{\intercal})$ which satisfies $z^{\intercal}\rM z = z^{\intercal}\rC z$ for every $z\in\mathbb R^{k}$.

The problem formulations in \cite{peyre2016gromov} and \cite{scetbon2022linear} are equivalent to this setup under the transformation $x=\vec(\rP)$ where $\rP\in\mathbb R^{N_1\times N_0}$ is a coupling matrix with entries $\rP_{lm}=\pi\left(\left\{\left(x_0^{(m)},x_1^{(l)}\right)\right\}\right)$ for some $\pi\in \Pi(\mu_0,\mu_1)$ and $\vec(\rP)\in\mathbb R^k$ is the vector obtained by stacking the columns of $\rP$. In the following, we let $\cK\coloneqq \{x\in\mathbb R^{k}:\rA x = b, x\geq 0\}$ denote the set of admissible couplings and write $\maxnorm{\rM}\coloneqq \max_{i,j=1}^k|\rM_{ij}|$

\section{Averaged Mirror Descent}\label{sec:algorithm}
Here and in the sequel, we assume that $\mu_0,\mu_1$ assign positive mass to each point in $\mathcal X_0,\mathcal X_1$.
As described previously, our proposed algorithm,  \Cref{algo:md-average}, consists of two steps. First, we solve a linearization of \Eqref{eq:EGWQP} at the previous iterate (which is simply an \gls*{eot} problem) and then we average the solution of the linearized problem with the previous iterate. The averaging step effectively introduces a step size and therefore prevents excessively large steps at a given iteration.

\begin{algorithm}[H]
\caption{Averaged Mirror Descent (AMD)}
\label{algo:md-average}
\SetAlgoLined
\SetAlgoNoEnd
\KwIn{initialization $x_0\in\cK$, maximum iteration number $J$, step sequence
$(\alpha_j)_{j\in\mathbb N}\subset(0,1]$}
\For{$j = 0, 1, \ldots, J-1$}{
    $\tilde x_j\gets$ approximate solution of
    $\min_{x\in\cK}\left\{x_j^{\intercal}\rM x-\varepsilon\mathsf H(x)\right\}$\;
    $x_{j+1}\gets (1-\alpha_j)x_j+\alpha_j \tilde x_j$\;
}
\end{algorithm} 
As noted above, the subproblem defining $\tilde{x}_j$ is simply an \gls*{eot} problem with the cost vector $\rM^{\intercal}x_j$ and $x_{j+1}$ is obtained by taking a convex combination of $x_j$ and $\tilde x_j$. We underscore that the \gls*{md} algorithm \cite[Algorithm 1]{scetbon2022linear} performs the same iterations as \Cref{algo:md-average} with the fixed step sequence $\alpha_j=1$ for each $j\in\mathbb N$, but the convergence properties of \gls*{md} for this problem for arbitrary cost matrices are unknown. Nevertheless, it remains the most popular \gls*{egw} solver for practical applications.  

As $\min_{x\in\cK}\left\{x_j^{\intercal}\rM x-\varepsilon\mathsf H(x)\right\}$ is an \gls*{eot} problem, it is solved using Sinkhorn iterations in practice (see \Cref{sec:Sinkhorn} for details on Sinkhorn's method). An important limitation of this algorithm is that it does not return an exact solution in a finite number of steps. Taking this into account, we first analyze  \Cref{algo:md-average} assuming that each subproblem is solved exactly and, later, account for inexactness.

\subsection{Convergence under exact solutions}

We now establish the non-asymptotic convergence properties of \Cref{algo:md-average} assuming that all \gls*{eot} subproblems are solved exactly. Our analysis hinges on the fact that \Cref{algo:md-average} is simply a conditional gradient type algorithm as analyzed in \cite{ghadimi2019conditional}. In effect, the algorithm studied in that work applies to the composite problem $\inf_{x\in X}\left\{f(x)+h(x)\right\}$ where $X\subset \mathbb R^d$ is a closed convex set, $f:X\to \mathbb R$ is  continuously differentiable with H{\"o}lder continuous gradient, and $h:X\to \mathbb R$ is $\mu$-strongly convex. Starting from some $x_0\in X$, their iterations read 
\[
\begin{aligned}
   \tilde x_j\gets \argmin_{u\in X}\left\{ \nabla f(x_j)^{\intercal}u + h(u)\right\},
   \quad
   x_{j+1}\gets (1-\alpha_j) x_{j}+\alpha_j \tilde x_j. 
\end{aligned}
\]
Notably, these iterations coincide exactly with \Cref{algo:md-average} by setting $X=\mathcal K$, $f(x)=\frac 12 x^{\intercal} \rM x$, and $h(x)=-\varepsilon \mathsf H(x)$ which is shown to be $\varepsilon$-strongly convex on the probability simplex $\Delta_k$ with respect to $\|\cdot\|_1$. This connection is central to the proof of convergence, see \Cref{app:md-exact-proof} for  full details.

\begin{theorem}[Convergence of \Cref{algo:md-average}]\label{prop:md-average-exact}
    Fix the step sequence $\alpha_j=\min\left\{\frac{\varepsilon}{4\maxnorm{\rM}},1\right\}$ and  assume that  $\min_{x\in\mathcal K}\left\{ \tilde x^{\intercal}\rM x - \varepsilon\mathsf H(x)\right\}$ can be solved exactly for each $\tilde x \in \mathcal K$. %
    Then, the iterates $\tilde x_j,x_j$ of \Cref{algo:md-average} satisfy
    \[
        \min_{j=0}^{J-1}\|x_{j}-\tilde x_j\|_1^2\leq \frac{8}{{3J\varepsilon}}\max\{{4\maxnorm{\rM}}/{\varepsilon},1\}{\left(f(x_0)-\varepsilon \mathsf H(x_0)-\inf_{\cK}\left\{f-\varepsilon \mathsf H\right\}\right)}.
    \]
\end{theorem}
We now discuss the interpretation of the convergence metric $\min_{j=0}^{J-1}\|x_j-\tilde x_j\|_1^2$.
\begin{remark}[Convergence criterion]
\label{rem:convergenceCriterion}
 Following the discussion surrounding Equation (2.5) in \cite{ghadimi2019conditional} and using the properties of the problem derived in \Cref{app:md-exact-proof}, if $\|x_j-\tilde x_j\|_1\leq \gamma/(\maxnorm{\rM}\sup_{x,y\in\mathcal K}\|x-y\|_1)$ for some $\gamma>0$, there exists $p\in\partial h(\tilde x_j)$ (the subdifferential of $h$ at $\tilde x_j$) for which 
$
\langle
-(p+\nabla f(\tilde x_j)),x-\tilde x_j \rangle \leq \gamma$
for each $x\in \mathcal K$. $\tilde x_j$ is then interpreted as a $\gamma$-approximate stationary point since a minimizer $\bar x\in\mathcal K$ for this problem is such that there exists $\bar p\in\partial h(\bar x)$ for which $\langle
-(\bar p+\nabla f(\bar x)),x-\bar x \rangle \leq 0$ for every $x\in\mathcal K$. We prove this assertion, show that $h$ is classically differentiable at each iterate, and provide further context in \Cref{sec:Clarke}.
\end{remark}

\subsection{Convergence under inexact solutions}

Next, we show that a similar convergence result holds when the \gls*{eot} subproblems are not solved exactly. The following convergence result is crucial for this analysis, see \Cref{sec:proofSinkhornConvergence} for its proof.
\begin{lemma}[Sinkhorn convergence]
\label{lem:convSinkhorn}
  Let  $x^{\star}$ be the unique solution of   $\mathsf{EOT}_{c}^{\varepsilon}(\mu_0,\mu_1)$. Then, for each $\delta>0$, there exists a number of steps $n$ depending on $\delta,\varepsilon, \|c\|_{\infty},\mu_0,$ and $\mu_1$ after which Sinkhorn's algorithm, as written in \Cref{algo:sinkhorn}, outputs a vector,  $\tilde x$, satisfying $\|\tilde x-x^{\star}\|_{1}\leq e^{\delta}-1$.  
\end{lemma}
An exact expression for $n$ can be found in Equation (28) of \cite{rioux2024entropic}.
Since each $\tilde x$ is a simplex vector (see \Cref{sec:Sinkhorn}) the infinity norm of the cost for the \gls*{eot} subproblems in \Cref{algo:md-average} is always bounded above by $\maxnorm{\rM}$ so that a fixed number of Sinkhorn iterations can be used to solve all problems to within a desired accuracy in the $1$-norm.

\begin{theorem}[Convergence under inexactness]\label{prop:md-average-inexact}
Fix $\delta>0$ and suppose that all \gls*{eot} problems are solved to an accuracy $e^{\delta}-1\leq \frac12$ in the $1$-norm and set $\alpha_j =\min\{\frac{\varepsilon}{2(\maxnorm{\rM}+\varepsilon)},1\}$.   
    Then, the iterates $\tilde x_j,x_j$ in \Cref{algo:md-average} satisfy
    \[
    \min_{j=0}^{J-1} \|x_j-\tilde x_j\|_1^2\leq\frac{4}{J \varepsilon}\frac{1}{\alpha_j}\left(f(x_0)-\varepsilon\mathsf H(x_0)-\inf_{\mathcal K}\{f-\varepsilon\mathsf H\} +\omega_\varepsilon(e^{\delta}-1)\right) +\frac{8}{\varepsilon} \omega_\varepsilon(e^{\delta}-1).
\]
     where $\omega_\varepsilon:p\in[0,2]\mapsto \maxnorm{\rM}p+\varepsilon\left(\left(\tfrac12\log(k-1)\right)p-\frac{p}{2}\log(\frac{p}{2})-(1-\frac{p}{2})\log(1-\frac{p}{2})\right)$.%
\end{theorem}
Notably, the optimization error decouples into a term scaling as $O(1/J)$ as in
\Cref{prop:md-average-exact} and an approximation error depending on $\omega_\varepsilon(e^{\delta}-1)$. As $e^{\delta}-1\leq \frac{3}{2}\delta$, since $e^{\delta}-1\leq \frac 12$ by assumption, and $\log(1-x)\geq \frac{-x}{1-x}$ for $x<1$ so that $-(1-x)\log(1-x)\leq x$ for $x\leq 1$ (the equality case corresponds to $0\leq 1$), we obtain $\omega_{\varepsilon}(e^{\delta}-1)\leq \frac{3}{2}\left(\maxnorm{\rM}\delta+\frac{\delta\varepsilon}{2}[\log((4(k-1))/(3\delta))+1]\right)$ so that $\omega_\varepsilon(e^{\delta}-1)=O(\delta\varepsilon\log(k/\delta))$ as $\delta\downarrow 0$ and the error can be made arbitrarily small. The proof of \Cref{prop:md-average-inexact} is included in \ref{app:proof-inexact}.
\begin{remark}[Inexact convergence criterion]
\label{rmk:inexactMetric}
    Due to inexactness, both $x_j$ and $\tilde x_j$ may not belong to $\mathcal K$ so that the interpretation of the convergence metric in  \Cref{prop:md-average-inexact} is more opaque than in \Cref{prop:md-average-exact}. Nevertheless, by leveraging the approximation results from the proof of \Cref{prop:md-average-inexact} we establish in \Cref{app:inexactIterates} that if $\|x_j-\tilde x_j\|_1\leq \tau$ for some $\tau>0$, there exists some $\tilde x^{\star}_j\in\mathcal K$ for which $\tilde x^{\star}_j$ is a $\maxnorm{\rM}\sup_{x,y\in\mathcal K}\|x-y\|_1(\tau +e^{\delta}-1)$-approximate stationary point and $\|\tilde x^{\star}_j-\tilde x_j\|_1\leq e^{\delta}-1$ provided that the \gls*{eot} subproblems are solved to within an accuracy of $e^{\delta}-1$ in the $1$-norm. 
\end{remark}

\section{Dual gradient method}\label{sec:dual-method}
Dual gradient methods \citep{rioux2024entropic,houry2026,rioux2026discrete} rely on a certain reformulation of the \gls*{gw} problem which trades off the quadratic program over $\mathcal K$ for nested optimization problems which depend on the coupling only in a linear manner.  
In effect, the recent work \cite{rioux2026discrete}  %
uses the fact that $\rM$ can always be expressed as 
$\rM = \rB_1^{\intercal}\rB_1 - \rB_0^{\intercal}\rB_0$
for matrices $\rB_0\in\mathbb{R}^{r_0\times k}$ and 
$\rB_1\in\mathbb{R}^{r_1\times k}$ with $r_0+r_1\leq k$ to establish that the generic \gls*{egw} problem  
$
\inf_{x\in\mathcal K}\left\{\frac 12 x^{\intercal}\rM x -\varepsilon \mathsf H(x)\right\}
$ is equivalent to the problem 
\begin{equation}
    \label{eq:variationalForm}
    \inf_{u\in\mathbb R^{r_0}}\sup_{v\in\mathbb R^{r_1}}\left\{ \frac12\|u\|^2-\frac 12 \|v\|^2+\inf_{x\in\mathcal K}\left\{\left( \rB_1^{\intercal}v-\rB_0^{\intercal} u \right)^{\intercal}x-\varepsilon\sH(x)  \right\} \right\}
\end{equation}
in the sense that their optimal values coincide and,  if $x^\star$ solves \Eqref{eq:EGWQP} then $(u^\star, v^\star) = (\rB_0 x^\star, \rB_1 x^\star)$ solves \Eqref{eq:variationalForm}. Conversely, if $(u^\star, v^\star)$ solve the inf-sup problem then there exists $x^\star \in \argmin_{x\in\mathcal K}\left\{\left( \rB_1^{\intercal}v-\rB_0^{\intercal} u \right)^{\intercal}x-\varepsilon\sH(x)  \right\}$ which solves \Eqref{eq:EGWQP} and $(u^\star, v^\star) = (\rB_0 x^\star, \rB_1 x^\star)$.  Proposition 8 in \cite{rioux2026discrete} shows that \Eqref{eq:variationalForm} can be written as
    \[
    \inf_{\mathbb R^{r_0}}\ell_{\varepsilon}\text{ for }
        \ell_{\varepsilon}:u\in\mathbb R^{r_0}\mapsto \frac 12 \|u\|^2 + \inf_{x\in\mathcal K}\left\{ \frac 12 x^{\intercal}\rB_1^{\intercal}\rB_1 x -u^{\intercal}\rB_0 x -\varepsilon \mathsf H(x)\right\}, 
    \]
    where $\ell_{\varepsilon}$ is smooth for any $\varepsilon>0$ with gradient  
    \[
        \nabla \ell_{\varepsilon}(u) = u-\rB_0x_{u} 
        \text{ where } x_u\in\argmin_{x\in\mathcal K}\left\{ \frac 12 x^{\intercal}\rB_1^{\intercal}\rB_1 x -u^{\intercal}\rB_0 x -\varepsilon \mathsf H(x)\right\}.
    \]
The same reference proposes to minimize $\ell_\varepsilon$ using projected gradient methods with and without acceleration and derives convergence bounds for the two methods under the inexactness of Sinkhorn's solutions. However, the convergence rates depend inversely on $\varepsilon$ and require setting the stepsize as a function of $\varepsilon$, leading to potentially small stepsizes. 

\begin{algorithm}
\caption{Dual gradient method}
\label{algo:dual}
\SetAlgoLined
\SetAlgoNoEnd
\KwIn{initialization $u_0=\rB_0x_0$ for $x_0\in\mathcal K$, maximum iteration count $J$}
\For{$j = 0, \ldots, J-1$}{
 $x_{u_j}\gets$  approximate solution of  $\min_{x\in\mathcal K}\left\{\frac 12 x^{\intercal}\rB_1^{\intercal}\rB_1x-u_j^{\intercal}\rB_0 x -\varepsilon\mathsf H(x)\right\}$\;

$u_{j+1}\gets u_j-(u_j-\rB_0 x_{u_j})$\; 
}
\end{algorithm}
We now introduce an inexact gradient method (\Cref{algo:dual}) for minimizing $\ell_\varepsilon$ which does not require setting a stepsize. However, each iteration requires solving the strongly convex problem $\min_{x\in\mathcal K}\left\{\frac 12 \|\rB_1x\|^2-u_j^{\intercal}\rB_0 x -\varepsilon\mathsf H(x)\right\}$ at least approximately.  \Cref{algo:md-average} can thus be thought of as taking coarse approximations of the gradient; this can explain why the original \gls*{md} method can fail on some examples, as it takes large steps in a direction which is not necessarily a gradient.  

\Cref{algo:dual} is related to the framework analyzed in \cite{houry2026} for CNT costs. Crucially, a CNT cost $\kappa_i$ on $\mathcal X_i\times \mathcal X_i$ $(i\in\{0,1\})$ is such that $\kappa_i(x,x')=\|\varphi_i(x)-\varphi_i(x')\|_{\mathcal H_i}^2$ where $\mathcal H_i$ is a Hilbert space with norm $\|\cdot\|_{\mathcal H_i}$ and $\varphi_i:\mathcal X_i\to\mathcal H_i$ is a continuous map. With this, the quadratic form $\iint |\kappa_0(x,x')-\kappa_1(y,y')|^2d\pi\otimes \pi(x,y,x',y')$ decomposes as $C(\mu_0,\mu_1)- 8 \|\int \varphi_0(x)\varphi_1(y)^{\intercal}d\pi(x,y)\|_{\mathrm{HS}}^2 -4\int \|\varphi_0(x)\|_{\mathcal H_0}^2 \|\varphi_1(y)\|_{\mathcal H_1}^2d\pi(x,y)$ for any $\pi\in\Pi(\mu_0,\mu_1)$, where $\|\cdot\|_{\mathrm{HS}}$ is the Hilbert-Schmidt norm and $C(\mu_0,\mu_1)$ depends only on moments of the marginals (see Lemma C.1 in \cite{houry2026}). It follows that CNT costs fall within the  variational framework for the special case $\rB_1 = 0$ when the marginals are finitely discrete, but with an added linear term which can be absorbed into the minimization over $\mathcal K$.

\subsection{Convergence under exact solutions}
We first derive convergence bounds for \Cref{algo:dual} assuming that we have access to exact solutions of $\min_{x\in\mathcal K}\bigl\{\frac 12 \|\rB_1x\|^2-u_j^{\intercal}\rB_0 x -\varepsilon\mathsf H(x)\bigr\}$. The proof follows by adapting the result of  \cite{houry2026} for CNT costs to the present dual formulation which holds for arbitrary costs, see \Cref{proof:thm:dual-exact}.

\begin{theorem}[Dual method convergence]\label{thm:dual-exact}
Fix $J\geq 1$, $\varepsilon>0$, and let $(u_j)_{j=0}^{J-1}$ be the iterates from \Cref{algo:dual} assuming that $\min_{x\in\mathcal K}\left\{\frac 12 \|\rB_1x\|^2-u_j^{\intercal}\rB_0 x -\varepsilon\mathsf H(x)\right\}$ can be solved exactly. Then, 
\[
    \min_{j=0}^{J-1} \|\nabla \ell_{\varepsilon}(u_j)\|^2 = 
    \min_{j=0}^{J-1} \|u_j-u_{j+1}\|^2 \leq \frac 2J\left(\ell_{\varepsilon}(u_0)-\inf_{\mathbb R^{r_0}}\ell_{\varepsilon}\right) 
\]
\end{theorem}

\begin{remark}[Comparison with \Cref{prop:md-average-exact}]
\label{rem:dual-criterion}
Compared with \Cref{prop:md-average-exact}, the convergence metric is straightforward; a point which is arbitrarily close to stationary will be met once $J$ is large enough. 
Furthermore, the rate does not depend on $\varepsilon^{-1}$ and so does not deteriorate as $\varepsilon$ becomes small, which is often the region of interest for practitioners. %
However, this algorithm requires a more complicated gradient call as noted previously.%
\end{remark}

\begin{remark}[The case of CNT costs]\label{rem:relation-to-md} As noted previously, the CNT costs analyzed in \cite{houry2026} correspond to the case $\rB_1=0$ in the current setting. In this case, the iterations coincide with the \gls*{md} algorithm, and the analysis is the same as in the aforementioned work. 
\end{remark}

Naturally, a point $u\in\mathbb R^{r_0}$ is called $\eta$-stationary for $\ell_{\varepsilon}$ with $\eta\geq 0$ if $\|\nabla \ell_{\varepsilon}(u)\|\leq \eta$. The notions of primal and dual approximate stationarity are related as follows, see \Cref{proof:thm:primalvsdual}.  
\begin{theorem}[Primal vs. dual stationarity]
\label{thm:primalvsdual} Fix $\varepsilon >0$, $\eta\geq 0$, and suppose that $\bar u\in\mathbb R^{r_0}$ is $\eta$-stationary for $\ell_{\varepsilon}$. Then, $x_{\bar u}$, the unique solution of $\argmin_{x\in\mathcal K}\left\{\frac 12 x^{\intercal}\rB_1^{\intercal}\rB_1x -\bar u^{\intercal}\rB_0 x -\varepsilon\mathsf H(x)\right\}$ is $2\|\rB_0\|_{1,2}\eta$-stationary for the primal problem as defined in \Cref{rem:convergenceCriterion}.
\end{theorem}

\subsection{Convergence under inexact solutions}

Since the solution to $\min_{x\in\mathcal K}\bigl\{\frac 12 \|\rB_1x\|^2-u_j^{\intercal}\rB_0 x -\varepsilon\mathsf H(x)\bigr\}$ is not available exactly  in practice,  we extend \Cref{thm:dual-exact} to the setting where we only have access to an inexact gradient oracle. Complete details can be found in \Cref{proof:thm:dual-inexact}.

\begin{theorem}[Dual convergence under inexactness]\label{thm:dual-inexact}
Fix $\tau>0$ and suppose that the problem  $\min_{x\in\mathcal K}\left\{\frac 12 \|\rB_1x\|^2-u_j^{\intercal}\rB_0 x -\varepsilon\mathsf H(x)\right\}$ can be solved to an accuracy $\tau$, i.e., $\|x_{u_j} - x_{u_j}^\star \|_1 < \tau$ for all $j$ where $x_{u_j}^\star$ denotes the true solution of the problem. Fix $J\geq 1$, $\varepsilon>0$, and let $(u_j)_{j=0}^{J-1}$ be the iterates obtained from \Cref{algo:dual}. Then, 
\[
    \min_{j=0}^{J-1} \|\nabla \ell_{\varepsilon}(u_j)\|^2  \leq \frac{2}{J}\left(\ell_{\varepsilon}(u_0)-\inf_{\mathbb R^{r_0}}\ell_{\varepsilon}\right) + \| \rB_0 \|_{1,2}^2 \tau^2 \text{ where $\| \rB_0 \|_{1,2} = \sup_{\| z\|_1 \leq 1} \|\rB_0 z \|_2$.}
\]
\end{theorem}
We first describe how to solve these quadratic subproblems and, afterwards, provide more context on the derived convergence rate.
\begin{remark}[Gradient computation] 
    As $\min_{x\in\mathcal K}\left\{\frac 12 \|\rB_1x\|^2-u_j^{\intercal}\rB_0 x -\varepsilon\mathsf H(x)\right\}$ is simply a regularized \gls*{qp}, we may apply dual gradient methods, e.g., Algorithms 5 and 6 in \cite{rioux2026discrete}, to solve it.\footnote{Both of these methods are iterative and only require solving \gls*{eot} problems at each iteration. This can be done in $O(N_0N_1)$ operations (up to $\log$ factors) via Sinkhorn's algorithm (\Cref{algo:sinkhorn}).} In fact, Propositions 11 and 12  of that work already establish how close the iterates get to the true solution so that the conditions of \Cref{thm:dual-inexact} can be shown to hold. However, the number of iterations to get a $\tau$-approximate solution scales as $\varepsilon^{-1}$. It is noteworthy that the proof of \Cref{thm:dual-inexact} does not apply for this convex subproblem as the inequalities are the ``wrong'' direction. 
\end{remark}

\begin{remark}[Approximation error]
    As in \Cref{prop:md-average-inexact}, the rate decouples into a term scaling as $O(1/J)$ and an approximation error which depends on the approximation accuracy $\tau$. The error and overall runtime then depend on the solver used to approximate the gradient. 

    While we cannot directly track $\|\nabla \ell_{\varepsilon}(u_j)\|^2$, $
    \|\nabla \ell_{\varepsilon}(u_j)\|\leq  \|\rB_0(x_{u_j}-x^{\star}_{u_j})\| + \|u_{j}-\rB_0 x_{u_j}\|\leq \|\rB_0\|_{1,2}\tau +\|u_{j+1}-u_j\|$
    so that terminating the algorithm when the approximate gradient $u_{j+1}-u_j=-(u_{j}-\rB_0 x_{u_j})$ has small magnitude is reasonable.
\end{remark}

\section{Experiments}\label{sec:experiments}

In this section we empirically validate the convergence rates derived in \Cref{sec:algorithm} \Cref{sec:dual-method}  across a variety of different settings. We also study the performance of these algorithms on a reconstruction of an example from \cite{houry2026}, where the original \gls*{md} algorithm fails to converge for a cost which is not CNT.

In the sequel, the implementation of the dual gradient method is identical to Algorithm 4 in \cite {rioux2026discrete} with the stepsize taken to $1$. We call this algorithm Variational (1) to emphasize this fact. The regularized convex quadratic subproblems used to approximate the gradients at each iteration are solved using Algorithm 6 from that work with the stepsize chosen per experiment. %
 We stress, however, that any other solver could be used to solve this inner problem as long as it can be shown to satisfy  $\|x_{u_j} - x^\star_{u_j} \|_1 < \tau$ for some $\tau > 0$.  We also compare Variational (1) to Algorithms 2 and 4 in \cite{rioux2026discrete} with theoretical stepsizes computed from a Lipschitz constant $L$ which we hence refer to as Accelerated variational ($L$) and Variational ($L$). The regularized quadratic subproblems are solved in the same manner. We report the values of $L$ per experiment in \Cref{app:theoretical-step}.

\begin{figure}[!htb]
    \centering
    \includegraphics[width=0.9\linewidth]{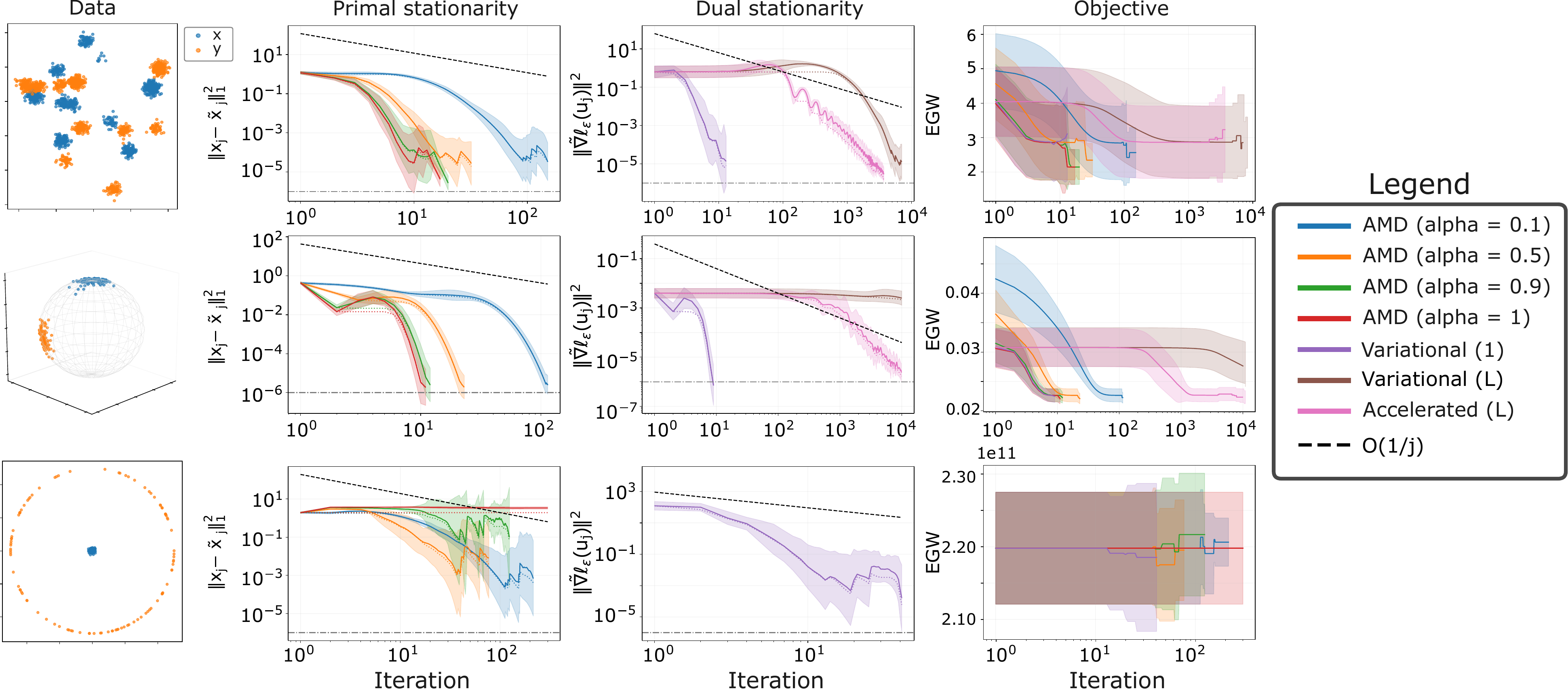}
    \caption{ Convergence of \gls*{amd} (\Cref{algo:md-average}) for $\alpha\in\{0.1,0.5,0.9,1\}$ ($\alpha=1$ recovering classical \gls*{md}), and the dual gradient method (\Cref{algo:dual}) across three datasets, which are visualized in the left-most panel. Primal stationarity records  $\|x_j-\tilde x_j\|_1^2$ for \gls*{amd}. Dual stationarity records $\|u_j-u_{j+1}\|^2$ for simple variational methods: Variational (1) and Variational ($L$), and $\|u_j-z_j\|^2$ for Accelerated ($L$) where $z_j$ is a second sequence used for momentum (see Algorithm 2 in \cite{rioux2026discrete}). The right-most plot shows the corresponding \gls*{egw} value. In both stationarity panels, solid lines show the per-iteration value and dotted lines its running minimum, with the $O(1/j)$ rate overlaid in black. Solid lines report the geometric mean for stationarity and the arithmetic mean for the objective over $10$ seeds; shaded regions indicate one standard deviation in log and non-log space respectively.}
    
    \label{fig:3-datasets}
\end{figure}

\begin{table}[htbp]
\centering
\small
\begin{tabular}{lcccccc}
\toprule
& \multicolumn{2}{c}{Gaussian mixture} & \multicolumn{2}{c}{Sphere} & \multicolumn{2}{c}{Gaussian-ring} \\
\cmidrule(lr){2-3} \cmidrule(lr){4-5} \cmidrule(lr){6-7}
Method & Runtime (s) & Conv. & Runtime (s) & Conv. & Runtime (s) & Conv. \\
\midrule
AMD ($\alpha=0.1$)   & $38.2 \pm 13.1$      & $10/10$ & $0.62 \pm 0.49$    & $10/10$ & $9.42 \pm 3.39$   & $8/10$ \\
AMD ($\alpha=0.5$)   & $7.69 \pm 2.74$      & $10/10$ & $0.117 \pm 0.090$  & $10/10$ & $4.66 \pm 3.76$   & $9/10$ \\
AMD ($\alpha=0.9$)   & $4.19 \pm 1.55$      & $10/10$ & $0.061 \pm 0.048$  & $10/10$ & $7.85 \pm 6.18$   & $6/10$ \\
AMD ($\alpha=1$)     & $3.66 \pm 1.41$      & $10/10$ & $0.054 \pm 0.042$  & $10/10$ & $15.10 \pm 3.19$  & $0/10$ \\
Variational (1)      & $9.06 \pm 2.94$      & $10/10$ & $0.128 \pm 0.113$  & $10/10$ & $106.1 \pm 57.8$  & $10/10$ \\
Variational ($L$)    & $2366 \pm 694$       & $8/10$  & $53 \pm 58$        & $0/10$  & ---               & --- \\
Accelerated variational ($L$) & $1155 \pm 300$ & $10/10$ & $92 \pm 74$   & $7/10$  & ---               & --- \\
\bottomrule
\end{tabular}
\caption{Mean solver wall-clock runtime in seconds ($\pm$ one standard deviation) over 10 seeds as well as the number of runs that reached the prescribed tolerance before reaching the maximum number of iterations. The measured runtime does not include the set up needed for the methods and in particular does not take into account the cost decompositions needed for both methods (for details see \Cref{app:implementation}).}
\label{tab:runtime}
\end{table}

Firstly, we illustrate convergence of \gls*{amd} for $\alpha\in\{0.1,0.5,0.9,1\}$ as well as the dual gradient method in three settings: a 10-component Gaussian mixture with $x\in\mathbb R^{10}$ and $y\in\mathbb R^{15}$ under the squared Euclidean cost which is CNT, following \cite{scetbon2022linear}, rotated point clouds on a sphere under the geodesic cost which is also CNT, and a Gaussian mapped to a ring both centered at $0$ under Euclidean cost raised to the sixth power which is not CNT. We use fixed values of $\alpha$ rather than the step sizes required by \Cref{prop:md-average-exact,prop:md-average-inexact} as they scale as $\varepsilon/\|\rM\|_{\max}$ (similarly to the variational methods of \citet{rioux2026discrete}) and are prohibitively small in all three settings. We report their values together with the corresponding stationarity plots in \Cref{app:theoretical-step} for completeness. 

The entries of $\mu_0,\mu_1$ are drawn uniformly at random and normalized to sum to one. In the sphere example $\mu_1$ is assigned the same weights as $\mu_0$, so that the two spaces are isomorphic. We set $\varepsilon=0.1$ for the Gaussian mixture, $\varepsilon=0.01$ for the sphere and $\varepsilon=1$ for the Gaussian-to-ring example. See \Cref{app:experimental-details} for dataset generation and other experimental details. 

\Cref{fig:3-datasets} shows the primal approximate stationarity $\|x_j-\tilde x_j\|_1^2$ for \gls*{amd}, the approximate dual gradient $\tilde\nabla\ell(u_j)$ for the variational methods, and the \gls*{egw} value across iterations. In the two CNT examples, \gls*{amd} with $\alpha<1$ converges more slowly than classical \gls*{md}, although $\alpha=0.5$ and $0.9$ require a comparable number of iterations. In the non-CNT Gaussian-ring example, classical \gls*{md} fails to converge in $300$ iterations for all 10 seeds whereas the smaller values of $\alpha$ are seen to converge on at least some of the runs. By contrast, Variational (1) converges for all experiments albeit slower than AMD. These methods are seen to abide by the theoretical $O(1/j)$ rate when they converge. %
As for the dual gradient methods, Variational~(1) is seen to converge much faster than Variational~(L) and its accelerated counterpart, as the stepsize $1/L$ used for those methods is very small. This also explains why those methods fail to converge in $10\,000$ iterations for some of the CNT examples; we thus omit these methods for the non-CNT example as they converge too slowly.

We underscore that all methods attain similar objective values. The large EGW values obtained in the Gaussian-ring example are due to the fact that the sixth-power cost produces a large constant term that hides the decrease of the (effective) objective. Since only the interaction term depends on the coupling (\Cref{app:interaction-objective}), we plot it separately in \Cref{fig:interaction-objective}, where both the decrease and the oscillations of classical \gls*{md} are clearly visible. The dual methods require a similar number of iterations to \gls*{amd} with $\alpha\in\{0.9,1\}$ on the CNT examples, and fewer than every \gls*{amd} variant on the non-CNT example. However, each iteration requires a more complex inner solve, so their total runtime is consistently higher (\Cref{tab:runtime}). We report the average number of Sinkhorn calls in \Cref{tab:sinkhorn-calls} and discuss the complexity of the methods in \Cref{app:dual-implementation}. 
Finally, we note that the same stopping tolerance is applied to different quantities for the primal and dual methods, so their iteration counts are not directly comparable.

\begin{figure}[t]
    \centering
    \includegraphics[width=.9\linewidth]{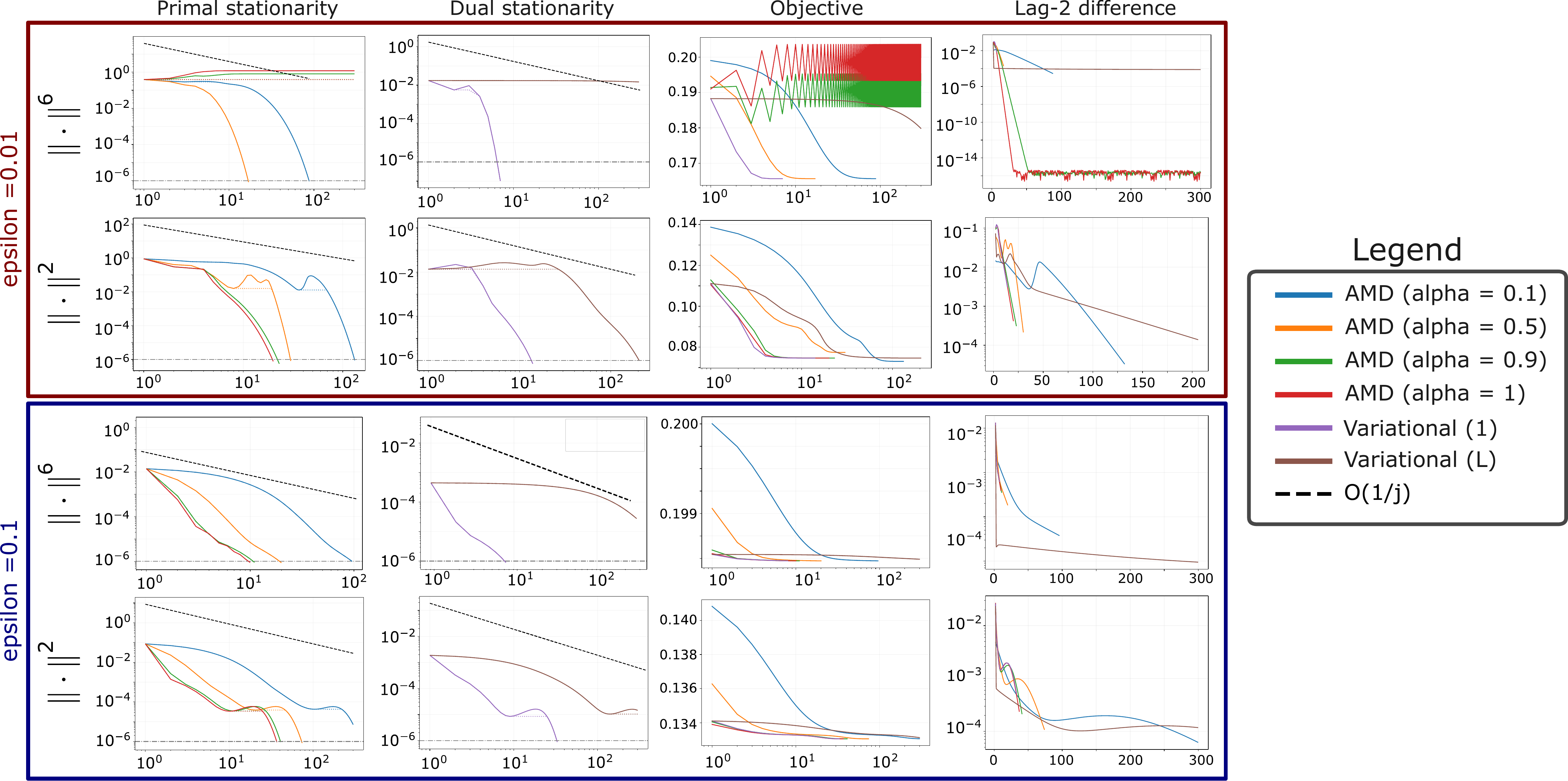}
    \caption{Results for the non-CNT example adapted from \cite{houry2026}. The plots track primal stationarity ($\|x_j-\tilde{x}_j\|_1^2$) for \gls*{amd} for $\alpha \in \{0.1, 0.5, 0.9, 1\}$, dual stationarity ($\|u_j - u_{j+1} \|^2$) for Variational (1) and Variational (L), the \gls*{egw} value and the lag-two difference between couplings at each iteration i.e. $\|x_j - x_{j+2}\|$ for all methods. The dotted line denotes the $O(1/j)$ reference curve. The top block shows the results for $\varepsilon=0.01$ and the bottom for  $\varepsilon=0.1$.}
    \label{fig:diagnostics-main}
\end{figure}

Next, we consider a reconstruction of the example introduced by
\cite{houry2026}, which compares squared Euclidean costs, which are
CNT, with sixth-power Euclidean costs, which are not CNT. Their example illustrates that the classical \gls*{md} iteration can oscillate between
two distinct couplings for non-CNT costs. We use two uniformly weighted point clouds containing 15 points each and study two regularization regimes, $\varepsilon=0.1 $ and $\varepsilon=0.01$. 
The details of the point cloud generation and implementation are described in \Cref{app:experimental-settings}.

\Cref{fig:diagnostics-main} reports the primal approximate stationarity for \gls*{amd}, the dual stationarity for Variational (1) and Variational ($L$), the \gls*{egw} value, and the lag-2 difference
$\|x_j-x_{j+2}\|_F$ for the various problems. When $\varepsilon=0.01$ with the sixth-power
cost, the objective oscillates for $\alpha\in\{0.9,1\}$, whereas it decreases
monotonically for $\alpha\in\{0.1,0.5\}$ and for both dual methods. For the
oscillating runs, the stationarity proxy remains bounded away from zero while the lag-2 difference becomes 
negligible, showing that the couplings oscillate with cycle two. This can be also seen in \Cref{app:additional-results}, where we visualize the different couplings. Every other run in \Cref{fig:diagnostics-main} decreases monotonically and meets the stopping criterion, including all runs at $\varepsilon=0.1$: the oscillation is therefore a joint effect of a
non-CNT cost and small regularization value. 

\section{Proofs of main results}\label{app:proofs}
First, we establish the strong convexity of the negative entropy on the simplex
$\Delta_k:=\{x\in\R^k: x\geq 0,\lVert x\rVert_1=1\}$ which, notably, contains $\cK$.

\begin{lemma}[Strong convexity of negative entropy]\label{lem:convexity-of-entropy} $-\mathsf H$ is
$1$-strongly convex on $\Delta_k$ with respect to $\lVert\cdot\rVert_1$. Thus, 
$h=-\varepsilon\mathsf H$ is $\varepsilon$-strongly convex on $\Delta_k$.
\end{lemma}

\begin{proof}
Fix $t\in(0,1)$ and $x,x'\in\Delta_k$, and write $m:=tx+(1-t)x'$ which again lies in $\Delta_k$ by convexity. Expanding the definitions of
$\mathsf H$ and of the KL divergence gives the identity
\[
-t\mathsf H(x)-(1-t)\mathsf H(x')
= t\,\mathrm{KL}(x\Vert m)+(1-t)\,\mathrm{KL}(x'\Vert m)-\mathsf H(m).
\]
Applying Pinsker's inequality gives,
\[
t\,\mathrm{KL}(x\Vert m)+(1-t)\,\mathrm{KL}(x'\Vert m)
\ \ge\ \tfrac12\Big(t\lVert x-m\rVert_1^{2}+(1-t)\lVert x'-m\rVert_1^{2}\Big),
\]
and since $x-m=(1-t)(x-x')$ and $x'-m=-t(x-x')$, the right-hand side equals
\[
\tfrac12\Big(t(1-t)^{2}+(1-t)t^{2}\Big)\lVert x-x'\rVert_1^{2}
=\tfrac12\,t(1-t)\lVert x-x'\rVert_1^{2}.
\]
Combining the two displays,
\[
-\mathsf H\big(tx+(1-t)x'\big)\ \le\ -t\mathsf H(x)-(1-t)\mathsf H(x')
-\tfrac12t(1-t)\lVert x-x'\rVert_1^{2},
\]
which proves $1$-strong convexity of $-\mathsf H$ on $\Delta_k$.
\end{proof}

\subsection{Proof of \Cref{prop:md-average-exact}}\label{app:md-exact-proof}

The proof will follow by applying the results from \cite{ghadimi2019conditional}. As noted
previously in \Cref{sec:algorithm}, \Cref{algo:md-average} coincides with the conditional
gradient type method studied in that work for the choices $X=\cK$,
$f(x)=\frac12x^{\intercal}\rM x$, and $h(x)=-\varepsilon\mathsf H(x)$. We note that the gradient
of $f$ satisfies
\[
\|\nabla f(x)-\nabla f(x')\|_\infty=\|\rM(x-x')\|_{\infty}
\leq\max_{i,j}|\rM_{i,j}|\,\|x-x'\|_1
\]
so that the gradient of $f$ is $1$-H\"older continuous (i.e., Lipschitz continuous) with respect
to $\lVert\cdot\rVert_{\infty}$ with constant $\maxnorm{\rM}\coloneqq\max_{i,j}|\rM_{i,j}|$. Moreover, by
Lemma~\ref{lem:convexity-of-entropy}, $h$ is $\varepsilon$-strongly convex on $\cK\subset\Delta_k$
with respect to $\|\cdot\|_{1}$, so that $\mu=\varepsilon$.

By Theorem 1 (b) in \cite{ghadimi2019conditional}, with
$\alpha_j=\min\left\{\frac{\varepsilon}{4\maxnorm{\rM}},1\right\}$,
\[
\sum_{j=0}^{J-1}\alpha_j\left(\nabla f(x_j)^{\intercal}(x_j-\tilde x_j)+h(x_j)-h(\tilde
x_j)\right)\leq\frac43\left(f(x_0)+h(x_0)-\min_{\cK}(f+h)\right).
\]
By Lemma 1 from the same reference,
\[
\|x_j-\tilde x_j\|_1^2\leq\frac{2}{\varepsilon}\left(\nabla f(x_j)^{\intercal}(x_j-\tilde
x_j)+h(x_j)-h(\tilde x_j)\right),
\]
so that, using the fact that $\alpha_j$ is constant,
\[
\min_{j=0}^{J-1}\|x_j-\tilde x_j\|_1^2\leq
\frac{8}{3J\varepsilon} \frac{1}{\alpha_j}
\left(f(x_0)-\varepsilon\mathsf H(x_0)-\min_{\cK}\{f-\varepsilon\mathsf H\}\right),
\]
which proves the claim upon inserting the expression for $\alpha_j^{-1}$.
\qed

\subsection{Proof of \Cref{prop:md-average-inexact}}\label{app:proof-inexact}
As noted in \Cref{sec:Sinkhorn}, the output of Sinkhorn's algorithm can be identified with an element of $\Delta_k$ upon vectorizing the matrix $\bm \Pi$, that is, $\tilde x_j\in\Delta_k$. It follows that the iterates $x_j$ of \Cref{algo:md-average} are all elements of $\Delta_k$ due to convexity of $\Delta_k$.

Throughout we write $\phi_{x_j}(x)=\nabla f(x_j)^{\intercal}x-\varepsilon\mathsf H(x)$ for the function
minimized in the exact update, $\tilde x_j^{\star}:=\argmin_{\cK}\phi_{x_j}$ and $L:=\maxnorm{\rM}$, so that $\tilde x_j$ is the
approximation of $\tilde x_j^{\star}$ returned by Sinkhorn's algorithm.

We now establish that $\phi_{x_j}$ admits a certain modulus of continuity over $\Delta_k$ which is a key ingredient in the proof of \Cref{prop:md-average-inexact}.

\begin{lemma}[Modulus of continuity] 
\label{lem:modulusContinuity}
    For any $x_j$ with $\|x_j\|_1=1$ and $x,x'\in \Delta_k$, 
    \[
    |\phi_{x_j}(x)-\phi_{x_j}(x')|\leq \omega_{\varepsilon}(\|x-x'\|_1) \text{ and } |(f(x)-\varepsilon\mathsf H(x))-(f(x')-\varepsilon\mathsf H(x'))|\leq \omega_{\varepsilon}(\|x-x'\|_1)
    \]
    where $\omega_\varepsilon:[0,2]\to \mathbb R$ is defined by 
    \[
    \omega_\varepsilon(p)=
        \maxnorm{\rM}p+ \varepsilon\left[\frac{1}{2} p\log(k-1)- \frac{1}{2}p\log\left(\frac{1}{2}p\right)- \left(1-\frac{1}{2}p\right)\log\left(1-\frac{1}{2}p\right)\right],
    \]
    which is increasing on $[0,1]$.
\end{lemma}
\begin{proof}
     Equation (11) in \cite{audenaert2007sharp} asserts that 
    \[
    \begin{aligned}
        \left|\mathsf H(x)-\mathsf H(x')\right|&\leq \frac{1}{2}\|x-x'\|_1\log(k-1)- \frac{1}{2}\|x-x'\|_1\log\left(\frac{1}{2}\|x-x'\|_1\right)
        \\
        &- \left(1-\frac{1}{2}\|x-x'\|_1\right)\log\left(1-\frac{1}{2}\|x-x'\|_1\right)
    \end{aligned} 
    \]
    for each non-negative vector with $1$-norm equal to $1$ of size $k$, noting that $\|x-x'\|_1\leq 2$. On the other hand, we have that 
    \[
        |\nabla f(x_j)^{\intercal} x -\nabla f(x_j)^{\intercal} x'| = |x_j^{\intercal}\rM( x - x')|\leq \|\rM x_j\|_{\infty} \|x-x'\|_1\leq \maxnorm{\rM}\|x-x'\|_1. 
    \]
    Altogether we have that 
    \[
    \begin{aligned}
        |\phi_{x_j}(x)-\phi_{x_j}(x')|&\leq\maxnorm{\rM}\|x-x'\|_1+   \frac{\varepsilon}{2}\|x-x'\|_1\log(k-1)- \frac{\varepsilon}{2}\|x-x'\|_1\log\left(\frac{1}{2}\|x-x'\|_1\right)
        \\
        &- \varepsilon\left(1-\frac{1}{2}\|x-x'\|_1\right)\log\left(1-\frac{1}{2}\|x-x'\|_1\right)
        \\&=\omega_\varepsilon(\|x-x'\|_1)
    \end{aligned} 
    \]
    which proves the first claim. Observe that 
    \[
    \begin{aligned}
        |f(x)-f(x')|=\left|\frac{1}{2}(x-x')^{\intercal}\rM x +\frac{1}{2}(x')^{\intercal}\rM (x-x') \right|&\leq \frac{1}{2}\left(\|\rM x\|_{\infty}+\|\rM x'\|_{\infty}\right)\|x-x'\|_1
        \\
        &\leq \maxnorm{\rM}\|x-x'\|_1
    \end{aligned} 
    \]
   so that the same argument as above establishes that 
   \[
         |(f(x)-\varepsilon\mathsf H(x))-(f(x')-\varepsilon\mathsf H(x'))|\leq \omega_{\varepsilon}(\|x-x'\|_1).
   \]
     
    To see that $\omega_\varepsilon$ is increasing on $[0,1]$, note that $p\in[0,1]\mapsto -p\log(p)-(1-p)\log(1-p)$ is increasing on $[0,1/2]$ and that the linear terms are evidently increasing. 
\end{proof}

With these results in hand, we are in position to prove \Cref{prop:md-average-inexact}.

\noindent\textbf{Proof of \Cref{prop:md-average-inexact}}
Since the gradient of $f$ is $L$-Lipschitz continuous with respect to $\|\cdot\|_1$ and
$x_{j+1}-x_j=\alpha(\tilde x_j-x_j)$,
\[
f(x_{j+1})\le f(x_j)+\alpha\nabla f(x_j)^{\intercal}(\tilde x_j-x_j)
+\frac{L\alpha^{2}}{2}\|\tilde x_j-x_j\|_1^{2}.
\]
As noted above $\tilde x_j\in\Delta_k$ and, since $x_0 \in \cK \subset \Delta_k$, $x_j \in \Delta_k$ for all $j$. 
Lemma~\ref{lem:convexity-of-entropy} yields that
\[
-\varepsilon\mathsf H(x_{j+1})\le-(1-\alpha)\varepsilon\mathsf H(x_j)-\alpha\varepsilon\mathsf H(\tilde x_j)
-\frac{\alpha(1-\alpha)}{2}\varepsilon\|x_j-\tilde x_j\|_1^{2}.
\]
Summing these two displayed equations, we obtain that
\[
\begin{aligned}
\alpha\left(\nabla f(x_j)^{\intercal}(x_j-\tilde x_j)+\varepsilon\mathsf H(\tilde x_j)
-\varepsilon\mathsf H(x_j)\right)&
\\&\hspace{-15em}\leq
\left(f(x_j)-\varepsilon\mathsf H(x_j)\right)-\left(f(x_{j+1})-\varepsilon\mathsf H(x_{j+1})\right)
+\frac{\alpha}{2}\left(L\alpha-(1-\alpha)\varepsilon\right)\|x_j-\tilde x_j\|_1^{2},
\end{aligned}
\]
and, since $\alpha=\frac{\varepsilon}{2(L+\varepsilon)}$,
$L\alpha-(1-\alpha)\varepsilon=-\varepsilon/2$, so that
\begin{equation}
\label{eq:convergenceMainBound}
\begin{aligned}
\frac{\alpha\varepsilon}{4}\|x_j-\tilde x_j\|_1^{2}
+\alpha\left(\nabla f(x_j)^{\intercal}(x_j-\tilde x_j)+\varepsilon\mathsf H(\tilde x_j)
-\varepsilon\mathsf H(x_j)\right)
&\\&\hspace{-8em}
\leq
\left(f(x_j)-\varepsilon\mathsf H(x_j)\right)-\left(f(x_{j+1})-\varepsilon\mathsf H(x_{j+1})\right).
\end{aligned}
\end{equation}
We now analyze the term $\nabla f(x_j)^{\intercal}(x_j-\tilde x_j)+\varepsilon\mathsf H(\tilde x_j)
-\varepsilon\mathsf H(x_j) = \phi_{x_j}(x_j)-\phi_{x_j}(\tilde x_j)$. Let $\tilde x_j^{\star}$ be the true minimizer of $\phi_{x_j}$ over $\mathcal K$ and 
let $z_{j+1} = (1-\alpha) z_j +\alpha \tilde x_j^{\star}\in\mathcal K$ for each $j\in\{0,\dots, J-1\}$ with the convention $z_0=x_0$. We prove by induction that $\|z_J-x_J\|_1\leq (e^{\delta}-1)\sum_{j=0}^{J-1}(1-\alpha)^j\alpha$. The base case is straightforward since $z_0=x_0$. Now, by the triangle inequality 
\[
\begin{aligned}
    \|z_{J+1}-x_{J+1}\|_1 &= \|(1-\alpha)(z_{J}-x_J)+\alpha(\tilde x_J^{\star}-\tilde x_{J})\|_1
    \\
    &\leq (e^{\delta}-1) \alpha (1-\alpha)\sum_{j=0}^{J-1}(1-\alpha)^j +\alpha (e^{\delta}-1)=(e^{\delta}-1)\alpha  
    \sum_{j=0}^{J}(1-\alpha)^j,
\end{aligned}
\]
where we have applied Lemma \ref{lem:convSinkhorn} to obtain that $\|\tilde x_J^{\star}-\tilde x_{J}\|_1\leq e^{\delta}-1$.

With these points, we note that 
\[
    \phi_{x_j}(x_j)- \phi_{x_j}(\tilde x_j) = \underbrace{\phi_{x_j}(x_j)- \phi_{x_j}(z_j)}_{(i)}+\underbrace{\phi_{x_j}(z_j)-\phi_{x_j}(\tilde x_j^{\star})}_{(ii)}+\underbrace{\phi_{x_j}(\tilde x_j^{\star})- \phi_{x_j}(\tilde x_j)}_{(iii)}.  
\]
Note that expression (ii) is non-negative since $\tilde x_j^{\star}$ is the global minimizer of $\phi_{x_j}$. On the other hand, since $x_j,z_j,\tilde x_j^{\star},$ and $\tilde x_j$ are elements of $\Delta_k$ as noted previously, Lemma \ref{lem:modulusContinuity} yields that 
\[
   \phi_{x_j}(x_j)- \phi_{x_j}(\tilde x_j)\geq -\omega_\varepsilon(\|x_j-z_j\|_1) - \omega_\varepsilon(\|\tilde x_j^{\star}-\tilde x_j\|_1)\geq -2\omega_\varepsilon(e^{\delta}-1), 
\]
where the final inequality follows from the fact that $\omega_\varepsilon$ is increasing on $[0,1]$ and that $e^{\delta}-1\leq \frac{1}{2}$.

Returning to \Eqref{eq:convergenceMainBound}, we obtain that  
\[
    \sum_{j=0}^{J-1} \|x_j-\tilde x_j\|_1^2\leq\frac{4}{\alpha \varepsilon}\left(f(x_0)-\varepsilon\mathsf H(x_0)-\left(f(x_{J})-\varepsilon\mathsf H(x_{J})\right)\right) +\frac{8}{\varepsilon}J \omega_\varepsilon(e^{\delta}-1). 
\]
Recalling $z_J\in\mathcal K$ from the previous step which satisfies $\|x_J-z_J\|_1\leq e^{\delta}-1$, we apply Lemma \ref{lem:modulusContinuity} again to obtain that
\[
f(x_J)-\varepsilon \mathsf H(x_J)\geq f(z_J)-\varepsilon \mathsf H(z_J) -\omega_\varepsilon(e^{\delta}-1)\geq \inf_{\mathcal K}\{f-\varepsilon\mathsf H\} -\omega_\varepsilon(e^{\delta}-1)
\]
whereby
\[
    \sum_{j=0}^{J-1} \|x_j-\tilde x_j\|_1^2\leq\frac{4}{\alpha \varepsilon}\left(f(x_0)-\varepsilon\mathsf H(x_0)-\inf_{\mathcal K}\{f-\varepsilon\mathsf H\} +\omega_\varepsilon(e^{\delta}-1)\right) +\frac{8}{\varepsilon}J \omega_\varepsilon(e^{\delta}-1). 
\]
We conclude by lower bounding the sum on the left hand side by $J$ times the minimum summand,
\[
    \min_{j=0}^{J-1} \|x_j-\tilde x_j\|_1^2\leq\frac{8(L+\varepsilon)}{J \varepsilon^2}\left(f(x_0)-\varepsilon\mathsf H(x_0)-\inf_{\mathcal K}\{f-\varepsilon\mathsf H\} +\omega_\varepsilon(e^{\delta}-1)\right) +\frac{8}{\varepsilon} \omega_\varepsilon(e^{\delta}-1).
\]
\qed

\subsection{Proof of \Cref{thm:dual-exact}}
\label{proof:thm:dual-exact}
From the proof of Lemma 2 in \cite{rioux2026discrete} 
    \[
        \ell_{\varepsilon}(u) = \frac 12 \|u\|^2+\inf_{x\in\mathcal K}\left\{\frac 12 x^{\intercal}\rB_1^{\intercal}\rB_1x -u^{\intercal}\rB_0 x -\varepsilon \mathsf H(x)\right\}
    \]
    and so, noting that $u_{j+1}=u_j - \nabla \ell_{\varepsilon}(u_j) = \rB_0 x_{u_j}$, 
    \[
    \begin{aligned}
        \ell_{\varepsilon}(u_j)- \ell_{\varepsilon}(u_{j+1}) &\geq \left(\frac 12\|u_j\|^2 +\frac{1}{2}\|\rB_1 x_{u_j}\|^2 - u_j^{\intercal}\rB_0 x_{u_j} -\varepsilon \mathsf H(x_{u_j})\right)
        \\
        &- \left(\frac 12\|u_{j+1}\|^2 +\frac{1}{2}\|\rB_1 x_{u_j}\|^2 - u_{j+1}^{\intercal}\rB_0 x_{u_j} -\varepsilon \mathsf H(x_{u_j})\right)
        \\
        & = \frac{1}{2}\|u_j\|^2 -u_j^{\intercal}\rB_0 x_{u_j} -\frac{1}{2}\|\rB_0 x_{u_j}\|^2 + \|\rB_0 x_{u_j}\|^2.
    \end{aligned}
    \]
    Conclude that $\ell_{\varepsilon}(u_j)- \ell_{\varepsilon}(u_{j+1})\geq \frac{1}{2}\|u_j-\rB_0x_{u_j}\|^2= \frac{1}{2}\|u_j-u_{j+1}\|^2$. 
    Summing this expression over the iterates gives 
    \[
        \frac 12 \sum_{j=0}^{J-1}\|u_j-u_{j+1}\|^2 \leq \ell_{\varepsilon}(u_0)- \ell_{\varepsilon}(u_J)\leq \ell_{\varepsilon}(u_0)-\inf_{\mathbb R^{r_0}} \ell_{\varepsilon}  
    \]

    which yields that 
    
    \[
    \min_{j=0}^{J-1}\|\nabla \ell_{\varepsilon}(u_j)\|^2= \min_{j=0}^{J-1}\|u_j-u_{j+1}\|^2 \leq \frac 2J \left(\ell_{\varepsilon}(u_0)-\inf_{\mathbb R^{r_0}} \ell_{\varepsilon}\right).
    \]
\qed

\subsection{Proof of \Cref{thm:dual-inexact}}
\label{proof:thm:dual-inexact}
Following the proof of \Cref{thm:dual-exact}, 
\[
\begin{aligned}
    \ell_{\varepsilon}(u_j)&=\frac 12 \|u_j\|^2 + \frac 12\|\rB_1x^{\star}_j\|^2 -u_j^{\intercal}\rB_0x^{\star}_j - \varepsilon\mathsf H(x_j^{\star})
    \\
    \ell_{\varepsilon}(u_{j+1})&\leq \frac 12 \|u_{j+1}\|^2 + \frac 12\|\rB_1x^{\star}_j\|^2 -u_{j+1}^{\intercal}\rB_0x^{\star}_j - \varepsilon\mathsf H(x_j^{\star}).
\end{aligned}
\]
Consequently, 
\[
\begin{aligned}
\ell_{\varepsilon}(u_j)-\ell_{\varepsilon}(u_{j+1})&\geq \frac 12 \|u_j\|^2- \frac 12 \|u_{j+1}\|^2 -(u_j-u_{j+1})^{\intercal}\rB_0x^{\star}_j
\\
&= \frac 12 \|u_j\|^2- \frac 12 \|u_{j+1}\|^2 -(u_j-u_{j+1})^{\intercal}\rB_0x^{\star}_j - \frac{1}{2}\|\rB_0x^{\star}_j\|^2 + \frac{1}{2}\|\rB_0x^{\star}_j\|^2
\\
&=\frac 12 \|u_j-\rB_0 x^{\star}_j\|^2- \frac 12 \|u_{j+1}-\rB_0 x^{\star}_j\|^2
\\
&= \frac 12 \|\nabla \ell_{\varepsilon}(u_j)\|^2- \frac 12 \|\rB_0 x_{u_j}-\rB_0 x^{\star}_j\|^2,
\end{aligned}
\]
where we recall that $u_j-\rB_0 x^{\star}_j=\nabla \ell_{\varepsilon}(u_j)$ is the true gradient of $\ell_{\varepsilon}$. Now, 
\[
    \|\rB_0 x_{u_j}-\rB_0 x^{\star}_j\|\leq \|\rB_0\|_{1,2}\|x_{u_j}-x^{\star}_j\|_1\leq \|\rB_0\|_{1,2}\tau,
\]
whereby
$
    \sum_{j=0}^{J-1} \|\nabla \ell_{\varepsilon}(u_j)\|^2 \leq 2\left(\ell_{\varepsilon}(u_0)-\ell_{\varepsilon}(u_J)\right) +J \|\rB_0\|^2_{1,2}\tau^2, 
$
and, since $\ell_{\varepsilon}(u_J)\geq \inf_{\mathbb R^{r_0}}\ell_{\varepsilon}$,
\[
\min_{j=0}^{J-1}  \|\nabla \ell_{\varepsilon}(u_j)\|^2 \leq \frac 2J\left(\ell_{\varepsilon}(u_0)-\inf_{\mathbb R^{r_0}}\ell_{\varepsilon}\right) + \|\rB_0\|^2_{1,2}\tau^2,
\]
proving the claim.
\qed

\subsection{Proof of \cref{thm:primalvsdual}}
\label{proof:thm:primalvsdual}
Suppose that $\bar u$ is an $\eta$-stationary point for $\ell_{\varepsilon}$ and some $\eta\geq 0$, that is, $\|\bar u-\rB_0 x_{\bar u}\|\leq \eta$ where $x_{\bar u}$ solves 
\[
    \inf_{x\in\mathcal K}\left\{\frac 12 x^{\intercal} \rB_1^{\intercal}\rB_1x-u^{\intercal}\rB_0x-\varepsilon \mathsf{H}(x)\right\}.
\]
Following the same arguments as in \Cref{sec:Clarke}, we necessarily have that 
\[
    \langle  \rB_1^{\intercal}\rB_1x_{\bar u}-\rB_0^{\intercal}\bar u + \nabla h(\bar x_{u}),x-\bar x_{u}\rangle \geq 0\text{ for each }x\in\mathcal K.
\] 
Applying the Cauchy-Schwarz inequality and the fact that  $\|\bar u-\rB_0 x_{\bar u}\|\leq \eta$,
\[
\langle -\rB_0^{\intercal}\bar u,x- x_{\bar u}\rangle  = \langle -\rB_0^{\intercal}(\bar u-\rB_0 x_{\bar u}+\rB_0 x_{\bar u}),x-x_{\bar u}\rangle \leq \langle -\rB_0^{\intercal}\rB_0 x_{\bar u},x-x_{\bar u}\rangle+  \|x-x_{\bar u}\|_1 \|\rB_0\|_{1,2}\eta. 
\] 
Since $\|x-x_{\bar u}\|_1 \leq 2$ as $x,x_{\bar u}$ are probability vectors, we obtain that, for each $x\in\mathcal K$, 
\[
     0\leq \langle  \rB_1^{\intercal}\rB_1x_{\bar u}-\rB_0^{\intercal}\bar u + \nabla h(x_{\bar u}),x-x_{\bar u}\rangle \leq  \langle  \rB_1^{\intercal}\rB_1x_{\bar u}-\rB_0^{\intercal}\rB_0x_{\bar u}+ \nabla h(x_{\bar u}),x-x_{\bar u}\rangle +2\|\rB_0\|_{1,2}\eta .
\]
Conclude that $\langle  -(\nabla h(x_{\bar u})+\nabla f(x_{\bar u})),x-x_{\bar u}\rangle \leq 2\|\rB_0\|_{1,2}\eta$ so that $x_{\bar u}$ is $2\|\rB_0\|_{1,2}\eta$-stationary for the primal problem according to the discussion in \Cref{rem:convergenceCriterion}. 
\qed

\section{Conclusion}
This work studied two provably convergent algorithms for solving Entropic Gromov--Wasserstein problems and established their connection to the mirror descent scheme  of \cite{peyre2016gromov} which is most commonly used by practitioners. First, we introduced Averaged Mirror Descent, a generalization of the mirror descent scheme for solving the \gls*{egw} problem which incorporates a step size parameter. This modification places \gls*{amd} within the conditional gradient framework of \cite{ghadimi2019conditional} which enabled us to obtain non-asymptotic convergence guarantees when exact iterates are assumed and to account for inexactness stemming from approximate \gls*{eot} solutions obtained via Sinkhorn's algorithm. Secondly, we showed that the dual gradient method from \cite{rioux2026discrete} is provably convergent with  a fixed regularization independent stepsize.
We have also discussed the connection between our dual method and the results obtained by \cite{houry2026} and showed that, for non-CNT costs, the MD scheme does not take true gradient steps explaining why it can fail to converge in practice on some examples.

In addition, we have empirically validated the convergence rates for both methods
across a variety of settings, demonstrating that \gls*{amd} attains performance
comparable to classical \gls*{md}. Moreover, the dual gradient method requires a similar number of steps as \gls*{amd}, but each iteration requires solving a regularized convex quadratic program in place of an \gls*{eot} problem so that the dual method is often slower in practice.

An important research direction stemming from this line of work is to speed up the gradient computation in the dual gradient method. While different gradient methods were proposed in \cite{rioux2026discrete}, their theoretically justified stepsizes become small as $\varepsilon$ decreases so that the convergence of these methods may be too slow for practical applications.  
\medskip

\textbf{AI use statement.} ChatGPT Astra and Claude Opus 5 were used for copyediting, to assist with figure creation and synthetic dataset generation, and to write auxiliary code for the simulations (plotting, multi-seed runs). They were not used to implement the main algorithms nor to  develop the theoretical results. 

\newpage

\bibliographystyle{plainnat}
\bibliography{bibliography}

\appendix

\section{Sinkhorn's algorithm}
\label{sec:Sinkhorn}
Sinkhorn's algorithm, \Cref{algo:sinkhorn}, solves an \gls*{eot} problem $\mathsf{EOT}_c^{\varepsilon}(\mu_0,\mu_1)$ between marginals supported on finite sets $\{x_0^{(i)}\}_{i=1}^{N_0}$ and $\{x_1^{(j)}\}_{j=1}^{N_1}$ respectively by using the fact that its solutions, $\pi^{\star}$, satisfy 
\begin{equation}
\label{eq:eotCoupling}
    \frac{\pi^{\star}(\{x^{(i)}_0,x^{(j)}_1\})}{\mu_0(\{x_0^{(i)}\})\mu_1(\{x_1^{(j)}\})} =  e^{\frac{u_i^{\star}+v_j^{\star}-c(x_0^{(i)},x_1^{(j)})}{\varepsilon}}\text{ for each }i=1,\dots, N_0,j=1,\dots,N_1,
\end{equation}
for some vectors $u^{\star}=(u_1^{\star},\dots,u_{N_0}^{\star})$ and  $v^{\star}=(v_1^{\star},\dots,v_{N_1}^{\star})$ which are unique up to adding a constant vector to $u^{\star}$ and subtracting the same constant vector from $v^{\star}$. With this representation, Sinkhorn's method fixes the cost matrix 
\[
   (\bm K)_{ij} = e^{\frac{-c(x_0^{(i)},x_1^{(j)})}{\varepsilon}}\text{ for each }i=1,\dots, N_0,j=1,\dots,N_1, 
\]
and the marginals $r_0=\left( \mu_0(\{x_0^{(1)}\}),\dots, \mu_0(\{x_0^{(N_0)}\})\right)$, $r_1=\left( \mu_1(\{x_1^{(1)}\}),\dots, \mu_1(\{x_1^{(N_1)}\})\right)$ and performs iterations  on a matrix with positive entries summing to $1$ which consist of iteratively fixing its marginals (the vectors obtained by left or right multiplication with the vector of ones) as follows.

\begin{algorithm}[H]
\caption{Sinkhorn Algorithm}
\label{algo:sinkhorn}
\SetAlgoLined
\SetAlgoNoEnd
\KwIn{stopping threshold $\gamma$, max iteration count $K$, cost matrix $\bm K$, marginal vectors $r_0,r_1$}
$u_0\gets \mathbbm{1}_{N_0}/N_0$\;
$k\gets 1$\;
\For{$k = 0, 1, \ldots, K-1$}{
 $v_k\gets r_1/(\bm K^{\intercal}u_{k-1})$\;
$u_k\gets r_0/(\bm K v_k)$\;
$\bm \Pi^k\gets \mathrm{diag}(u_k)\bm K\mathrm{diag}(v_k)$\;
\If{ $\|(\bm \Pi^k)^{\intercal}\mathbbm{1}_{N_0} -r_1\|_1<\gamma$ }{\text{break}\;}
$k\gets k+1$\;
}
\end{algorithm} 
The convergence properties of \Cref{algo:sinkhorn} are well-established as noted in Lemma \ref{lem:convSinkhorn}. Although the output, $\bm \Pi$, of \Cref{algo:sinkhorn} does not necessarily satisfy both marginal constraints simultaneously, $\bm \Pi$ is evidently positive by construction and satisfies the marginal constraint of $\bm \Pi \mathbbm 1_{N_1}=r_0$ by construction. It follows that $\bm \Pi$ can be thought of as a probability matrix and hence can be identified with a probability distribution on $\mathcal X_0\times \mathcal X_1$. 

We conclude this section with the proof of Lemma \ref{lem:convSinkhorn} now that the setting is clear. 

\subsection{Proof of Lemma \ref{lem:convSinkhorn}}
\label{sec:proofSinkhornConvergence}
The proof of Proposition 8 in \cite{rioux2024entropic} establishes that, 
for each $\delta>0$, there exists a number of steps $n$ depending on $\delta,\varepsilon, \|c\|_{\infty},\mu_0,$ and $\mu_1$ after which Sinkhorn's algorithm, as written in \Cref{algo:sinkhorn}, outputs a vector,  $\tilde x$, satisfying $\mathsf d(\tilde x,x^{\star})\leq \delta$ where $\mathsf d$ is a certain pseudometric. The proof of Lemma 26 in the same reference shows that if $\mathsf d(\tilde x,x^{\star})\leq \delta$, 
\[
   |\tilde x_i-x^{\star}_i|\leq \tilde x_i(e^{\delta}-1)\text{ for each }i\in\{1,\dots,k\}.
\]
Conclude that $\|\tilde x-x^{\star}\|_1\leq e^{\delta}-1$ as claimed.
\qed

\section{Additional details on the convergence metrics}
This section provides additional context for the convergence metrics used in \Cref{prop:md-average-exact,prop:md-average-inexact}. We begin with the case of exact iterates and later comment on the case where iterates are inexact. 

\subsection{Exact iterates}
\label{sec:Clarke}

We first establish that every local or global minimizer of $f+h$ over $\mathcal K$ has strictly positive entries. We denote this property by $x\in\mathbb R^k_{>0}$.

\begin{lemma}
\label{lem:relativeInteriorSolutions}
    Every local or global minimizer of $f+h$ over $\mathcal K$ lies in $\mathbb R^k_{>0}$.
\end{lemma}
\begin{proof}
 By definition $\bar x$ is a local minimizer of $f+h$ over $\mathcal K$ if and only if there exists a neighborhood, $N$, of $\bar x$ for which $(f+h)(x)\geq (f+h)(\bar x)$ for every $x\in N\cap \mathcal K$. It follows that, 
\begin{equation}
\label{eq:directionalDerivative}
    \liminf_{t\downarrow 0}\frac{(f+h)(\bar x+t(x-\bar x))-(f+h)(\bar x)}{t}= \liminf_{t\downarrow 0}\frac{(f+h)(x_t)-(f+h)(\bar x)}{t}\geq 0
\end{equation}
for any choice of $x\in N\cap \mathcal K$ noting that $x_t= (1-t)\bar x + t x\in \mathcal K$ for each $t\in[0,1]$ by convexity of $\mathcal K$ so that limit is well-defined and $x_t\in N$ once $t$ is sufficiently small. Since $f$ is smooth,
\[
    \lim_{t\downarrow 0}\frac{f(x_t)-f(\bar x)}{t}=\bar x^{\intercal}\rM(x-\bar x).
\]
As for $h$, observe that 
\[
\begin{aligned}
 {-\mathsf H(x_t)+\mathsf H(\bar x)} &= \sum_{i=1}^{k} \left( \bar x_i+t(x_i-\bar x_i)\right)\log \left( \bar x_i+t(x_i-\bar x_i)\right) - \bar x_i\log \left( \bar x_i\right)
 \\
 &=\sum_{i=1}^{k} \left( \bar x_i\left(\log \left( \bar x_i+t(x_i-\bar x_i)\right)-\log \left( \bar x_i\right)\right)+t(x_i-\bar x_i)\log \left( \bar x_i+t(x_i-\bar x_i)\right) \right).
 \end{aligned}
\]
Now, if $\bar x_I=0$ for some $I\in\{1,\dots,k\}$ we obtain that
\[
    \bar x_I\left(\log \left( \bar x_I+t(x_I-\bar x_I)\right)-\log \left( \bar x_I\right)\right)+t(x_I-\bar x_I)\log \left( \bar x_I+t(x_I-\bar x_I)\right) = tx_I\log \left( t x_I\right). 
\]
and
$\lim_{t\downarrow 0} x_I\log(t x_I)=-\infty$ if $x_I>0$. As noted in \Eqref{eq:eotCoupling}, there always exists some $x'\in\mathcal K\cap \mathbb R^k_{>0}$ and, by convexity of $\mathcal K$, $x=(1-\alpha) \bar x+\alpha x'\in (\mathcal K\cap N)\cap \mathbb R^k_{>0}$ for some $\alpha \in (0,1)$ whereby 
\[
 \lim_{t\downarrow 0}\frac{(f+h)(\bar x+t(x-\bar x))-(f+h)(\bar x)}{t}=-\infty
\]
for this choice of $x$, contradicting optimality of $\bar x$.

Conclude that a local minimizer, $\bar x$, must also have strictly positive entries and hence this proviso must also be met for a global minimizer.    
\end{proof}

With this, we obtain a necessary condition for optimality for the problem $\inf_{x\in\mathcal K}\{f+h\}$.
\begin{proposition}
    \label{prop:optimalityCondition}
    Suppose that $\bar x$ is a global minimizer of $f+h$ over $\mathcal K$. Then, $h$ is differentiable at $\bar x$  with derivative $\nabla h(\bar x)$ defined componentwise as
\[
    (\nabla h(\bar x))_i = \varepsilon(1+\log(\bar x_i)) \text{ for }i\in\{1,\dots,k\}   
\] 
and $\bar x$ must satisfy the property that 
    \begin{equation}
    \label{eq:necessaryCondition}
        \langle \nabla f(\bar x)+\nabla h(\bar x),x-\bar x \rangle \geq 0\text{ for each } x\in \mathcal K.
    \end{equation}
\end{proposition}
\begin{proof}
    By Lemma \ref{lem:relativeInteriorSolutions} $\bar x$ is an element of $\mathcal K\cap \mathbb R_{>0}^{k}$ so that the formula for the gradient of $h$ at $\bar x$ follows by a computation. Since $\bar x$ is a global minimizer, \Eqref{eq:directionalDerivative}  must hold at each $x\in\mathcal K$, that is, 
    $
        \langle \nabla f(\bar x)+\nabla h(\bar x),x-\bar x\rangle \geq 0
    $ for each $x\in\mathcal K$
    as claimed. 
\end{proof} 

With this result, it is natural to define the notion of a near stationary point in terms of how much \Eqref{eq:necessaryCondition} is violated. As noted in the main text, we can thus call $\tilde x\in\mathcal K\cap \mathbb R_{>0}^k$ $\gamma$-approximately stationary for this problem if  
$
        \langle \nabla f(\tilde x)+\nabla h(\tilde x),x-\tilde x\rangle \geq -\gamma
    $
    for some $\gamma>0$ and for each $x\in\mathcal K$. 

    We conclude this section by recalling that the solution of each \gls*{eot} subproblem is an element of $\mathbb R_{>0}^k$ (see \Eqref{eq:eotCoupling}) so that the approximate stationarity condition above makes sense at each $\tilde x_j$ as the derivatives are well-defined.  

    \subsection{Inexact iterates}
    \label{app:inexactIterates}
    As noted in Remark~\ref{rmk:inexactMetric}, the convergence metric $\min_{j=0}^{J-1}\|x_j-\tilde x_j\|_1^2$ is more opaque in the inexact setting since $\tilde x_j$ is not an exact solution of the \gls*{eot} subproblem and is not an element of $\mathcal K$.  

    \begin{proposition} Fix $\delta>0$ and suppose that every EOT subproblem is solved to within an accuracy of $e^{\delta}-1$ in the $1$-norm. Furthermore,
        let $\tilde x_j^{\star}$ be the unique solution of $\min_{x\in\mathcal K}\{x_j^{\intercal}\rM x-\varepsilon\mathsf H(x)\}$ for each $j\in\{0,\dots,J-1\}$ and hence satisfies $\|\tilde x_j^{\star}-\tilde x_j\|_1\leq e^{\delta}-1$. Then, if $\|x_j-\tilde x_j\|_1\leq \tau$ it holds that 
        \[
            \langle - \nabla f(\tilde x_j^{\star})-\nabla h(\tilde x_j^{\star}),x-\tilde x_j^{\star}\rangle \leq \maxnorm{\rM}\sup_{x,y\in\mathcal K}\|x-y\|_1\left(\tau+e^{\delta}-1\right),
        \]
        that is, $\tilde x_j^{\star}$ is an $\maxnorm{\rM}\sup_{x,y\in\mathcal K}\|x-y\|_1\left(\tau+e^{\delta}-1\right)$-approximate stationary point which is a distance at most $e^{\delta}-1$ from $\tilde x_j$.
    \end{proposition}
    \begin{proof}
        Since $\tilde x^{\star}_j$ solves $\min_{x\in\mathcal K}\{x_j^{\intercal}\rM x-\varepsilon\mathsf H(x)\}$ we have from the same logic as the proof of Lemma \ref{lem:relativeInteriorSolutions} and Proposition \ref{prop:optimalityCondition} that $\tilde x^{\star}_j\in\mathbb R^k_{>0}$ and  
        \[
           \langle -(\nabla f(x_j)+\nabla h(\tilde x^{\star}_j)),x-\tilde x^{\star}_j\rangle \leq 0\text{ for each }x\in\mathcal K.
        \]
        Since $f$ has $\maxnorm{\rM}$-Lipschitz continuous gradient in the $\infty$-norm, 
        \[
        \begin{aligned}
            \langle\nabla f(x_j)-\nabla f(\tilde x_j^{\star}),x-\tilde x_j^{\star}\rangle&\leq \|\nabla f(x_j)-\nabla f(\tilde x_j^{\star})\|_{\infty}\|x-\tilde x_j^{\star}\|_1 
            \\
            &\leq \sup_{x,y\in\mathcal K}\|x-y\|_{1}\maxnorm{\rM}\|x_j-\tilde x_j^{\star}\|_1
            \\
            &\leq \sup_{x,y\in\mathcal K}\|x-y\|_{1}\maxnorm{\rM}(\tau +e^{\delta}-1).
        \end{aligned} 
        \]
        Combining the two displayed equations above we obtain that 
        \[
             \langle -(\nabla f(\tilde x_j^{\star})+\nabla h(\tilde x^{\star}_j)),x-\tilde x^{\star}_j\rangle\leq \sup_{x,y\in\mathcal K}\|x-y\|_{1}\maxnorm{\rM}(\tau +e^{\delta}-1) 
        \]
        which shows that $\tilde x_j^{\star}$ is an $\sup_{x,y\in\mathcal K}\|x-y\|_{1}\maxnorm{\rM}(\tau +e^{\delta}-1)$-approximate stationary point of the EGW problem.
    \end{proof}

\section{Experimental details}\label{app:experimental-details}
\subsection{Dataset generation}
\label{app:dataset-generation}

\paragraph{Gaussian mixture.}
We independently generate $1000$ points from two $10$-component Gaussian
mixtures, one in $\mathbb{R}^{10}$ and the other in $\mathbb{R}^{15}$.
Component weights are sampled from a symmetric
$\operatorname{Dirichlet}(2)$ distribution, component centres from
$\mathcal{N}(0,9I)$, and each component has standard deviation $0.45$.
For visualization, each point cloud is projected onto its first two principal components. The cost is set to be the squared Euclidean distance. Both cost matrices are rescaled by the common factor $3/\max\{\mathrm{median}(\mathrm{K}_0),\mathrm{median}(\mathrm{K}_1)\}$ so that
they are on a comparable scale; since the same factor is applied to both, this is equivalent to rescaling $\varepsilon$.

\paragraph{Spherical point clouds.}
We sample $100$ points on $\mathbb{S}^2$ from a von Mises--Fisher
distribution with mean direction $(0,0,1)^\top$ and concentration
$\kappa=40$. The second point cloud is obtained by applying a random rotation to the first. The cost is chosen to be the geodesic distance.

\paragraph{Gaussian--ring.}
We generate two point clouds of $100$ points in $\mathbb{R}^2$. The first is
sampled independently from $x_i\sim\mathcal{N}(0,0.1^2I_2)$.
For the second, we draw
$\theta_i\sim\operatorname{Unif}[0,2\pi)$ independently and set
$y_i=5(\cos\theta_i,\sin\theta_i),$ so that the points are uniformly distributed on a circle of radius $5$ centered at the origin. The cost is set to be the Euclidean distance raised to the sixth power.

\paragraph{Example from \cite{houry2026}}

We reconstruct the example from \cite{houry2026} using
two fixed clouds of $15$ points in $\mathbb R^2$.
The pixel coordinates were  approximated by an LLM from the
published visualization, as the exact figure reproducing code was not publicly available. The sampled points are defined as
\[
x_i=(q_i^X,-r_i^X)/150,\qquad
y_i=(q_i^Y,-r_i^Y)/150,
\]
using the entries in the following coordinate table and the weights are taken to be uniform.

\[
\begin{array}{c|rr|rr}
i&q_i^X&r_i^X&q_i^Y&r_i^Y\\ \hline
1 &105&131&244&48\\
2 &120&138&260&43\\
3 &98 &197&279&50\\
4 &144&202&282&60\\
5 &126&225&267&74\\
6 &114&238&230&88\\
7 &133&242&290&87\\
8 &145&230&284&103\\
9 &187&219&304&109\\
10&180&250&291&119\\
11&203&234&255&119\\
12&233&266&245&140\\
13&211&150&344&115\\
14&214&139&344&130\\
15&122&238&313&25
\end{array}
\]

To verify that the approximation was accurate we run the \texttt{EntropicGW} function for the squared Euclidean cost and the Euclidean cost raised to the sixth power from the repository supplied by \cite{houry2026} and obtained the same oscillatory effect for the sixth power cost.

\subsection{Implementation details and computational complexity}\label{app:implementation}
\subsubsection{Implementation of \gls*{amd}}
\label{app:amd}
Our implementation of \gls*{amd} adapts the \gls*{md} method used in \citet{scetbon2022linear} which leverages low-rank factorizations of the cost matrices to speed up computations, see Algorithm~2 in \citet{scetbon2022linear}). Once the low-rank factorizations of the costs are given (we discuss similar decompositions in the following section), the complexity of this algorithm is O($N_0N_1 SJ$) up to log factors where $J$ denotes the number of iterations needed to reach approximate stationarity and $S$ denotes the maximum number of Sinkhorn's iterations. We note that the averaging step can impact $J$ which can be both larger and smaller than for the classical MD depending on the exact experimental setting as we have shown in \Cref{sec:experiments}.

\subsubsection{Implementation of the variational methods}
\label{app:dual-implementation}

\paragraph{Matrix decomposition.}
In order to implement the variational methods we first need to find the decomposition $\rM = \rB_1^\intercal \rB_1 -\rB_0^\intercal \rB_0$.
 A generic approach to construct $\rB_0$ and $\rB_1$ described in \cite{rioux2026discrete}, is to compute the eigendecomposition of $\rM$ and let $\rB_1=\mathrm{diag}(\sqrt{\lambda_1},\dots,\sqrt{\lambda_{r_1}}) \mathbf V_1$, where $(\lambda_1,\dots,\lambda_{r_1})$ is the collection of all positive eigenvalues of $\rM$ and $\mathbf V_1\in\mathbb R^{r_1\times k}$ has $i$-th row given by the (transpose of the) $i$-th eigenvector $v_i$. $\rB_0$ is constructed similarly using the absolute value of the negative eigenvalues. Eigenvalues with
absolute value at most $10^{-10}$ are discarded. As noted by \cite{rioux2026discrete} for the case $p=2$, 
 it suffices to solve the regularized quadratic program,  \eqref{eq:EGWQP}, with $\rM=-4\mathrm K_0\otimes \mathrm{K}_1$, where $\mathrm K_0\otimes \mathrm{K}_1$ is the Kronecker product of
 the pairwise cost matrices $\mathrm{K}_0\in\mathbb R^{N_0\times N_0},\mathrm{K}_1\in\mathbb R^{N_1\times N_1}$ whose $ij$-th entries are given by $\kappa_0\left(x_0^{(i)},x_0^{(j)}\right)$ and $\kappa_1\left(x_1^{(i)},x_1^{(j)}\right)$ respectively. They note that the complexity of this decomposition is $O(\max\{N_0^3,N_1^3\})$.  They further note that the memory efficiency of the implementation can be enhanced using some properties of Kronecker product and hence in all of our runs we use the memory-efficient implementation as described in \cite{rioux2026discrete} which applies $\rB_0$, $\rB_1$ and their transposes through matrix products using their derived memory efficient formulas without explicitly forming the full matrices  $\rB_0$, $\rB_1$.

\paragraph{Hyperparameters.} For all of the variational solvers, the inner maximization is set to a simple gradient method (Algorithm 6 in \cite{rioux2026discrete}) which requires solving  an EOT problem at each iteration. The number of gradient ascent steps and the step size varies by experiment and is described in the \Cref{app:experimental-settings}. The gradient ascent is terminated early if the squared norm of the gradient is below $10^{-6}$. For the Variational (1) method we set the outer minimization stepsize to $1$, which is theoretically justified by our convergence analysis. For the Variational ($L$) and Accelerated variational ($L$) we set the outer step sizes based on the value of $L_{\text{var}} =\frac{\lambda_{\max}(\rB_0^{\intercal}\rB_0)}{\lambda_{\min}(\rB_1^{\intercal}\rB_1)+\varepsilon} + 1$ for Variational ($L$) and the update sequences for Accelerated variational ($L$) are set using $L_{\text{acc}} = \max\left\{1,\frac{\lambda_{\max}(\rB_0^{\intercal}\rB_0)}{\lambda_{\min}(\rB_1^{\intercal}\rB_1)+\varepsilon}-1\right\}$ where $\lambda_{\min}$ and $\lambda_{\max}$ denote the smallest and largest eigenvalues respectively. These values come from Proposition 10 and Theorem 5 in \cite{rioux2026discrete}. We initialize all solvers from the product coupling $\mu_0 \otimes \mu_1$. %

\paragraph{Complexity.} Assuming that $\rB_0$ and $\rB_1$ are available, \cite{rioux2026discrete} report that the overall computational complexity of their simple gradient method with a first order oracle used for the inner solve scales as $O(J_{\text{outer}}J_{\text{inner}}(S+\max\{r_0,r_1\})N_0N_1)$, where $J_{\text{outer}}$ denotes the number of outer iterations, $J_{\text{inner}}$ denotes the number of inner oracle iterations and $S$ is the maximum number of Sinkhorn's iterations. The $r_0$ and $r_1$ denote the dimensions of $B_0$ and $B_1$ respectively.

\subsection{Experimental details}
\label{app:experimental-settings}

\paragraph{Sinkhorn implementation and parameters.}
All EOT problems are solved using POT \citep{flamary2021pot}. We either use the standard \texttt{sinkhorn\_knopp} or  \texttt{sinkhorn\_epsilon\_scaling} which is more stable, but slower.
For the three datasets, the termination threshold for these methods is set to
$10^{-9}$. For \texttt{sinkhorn\_knopp} we set the maximum number of iterations to $1000$ and, for \texttt{sinkhorn\_epsilon\_scaling},  a maximum of $20\,000$ continuation iterations
with POT's default of $100$ inner scaling iterations
per continuation step is used. 

The settings used for the methods are compiled in  \Cref{tab:experiment-settings}.

\begin{table}[htbp]
\centering
\small
\begin{tabular}{lccccc}
\toprule
Data & $\varepsilon$ & $J_{\text{outer}}$ & $J_{\text{inner}}$ & Inner stepsize & Sinkhorn method \\
\midrule
Gaussian mixture & $0.1$  & $10\,000$ & $100$ & $0.5$  & \texttt{sinkhorn\_knopp} \\
Sphere           & $0.01$ & $10\,000$ & $100$ & $0.5$  & \texttt{sinkhorn\_knopp} \\
Gaussian-ring   & $1$    & $300$     & $200$ & $0.05$ & \texttt{sinkhorn\_epsilon\_scaling} \\
\bottomrule
\end{tabular}
\caption{Settings for the experiments. $J_{\mathrm{outer}}$ is the total number of iterations used for \gls*{amd} (\Cref{algo:md-average}) and the dual gradient method (\Cref{algo:dual}). $J_{\mathrm{inner}}$ is the maximum  number of gradient ascent iterations used to solve each convex quadratic subproblem in \Cref{algo:dual} and inner stepsize refers to the stepsize used for this gradient ascent algorithm. Sinkhorn method refers to what method was used to solve the \gls*{eot} problems.}  %
\label{tab:experiment-settings}
\end{table}

The \texttt{sinkhorn-epsilon-scaling} was used for the Gaussian-ring example, as \texttt{sinkhorn-knopp} encounters division by zero errors.

\paragraph{Stopping criteria.}
 We terminate \gls*{amd} when $\| \tilde{x}_j - x_j \|^2_1 < 10^{-6}$ or when the maximum number of iterations is reached. For all experiments except for the Gaussian-ring example this tolerance is always met before reaching the maximum number of iterations. For the Gaussian-ring examples 8/10 runs for $\alpha = 0.1$ reach the stopping criterion within the prescribed number of iterations, 9/10 runs for $\alpha = 0.5$, 6/10 runs for $\alpha = 0.9$ and none of the classical mirror descent ($\alpha=1$) runs reach the prescribed tolerance (see \Cref{tab:runtime} for details).

All of the dual gradient methods terminate once the requisite metric is less than $10^{-6}$ or a maximum number of iterations is met. For the Variational (1) and Variational (L) methods, the metric of interest is $\|\tilde{\nabla} \ell(u_j)\|$ and, for the Accelerated variational (L) method, it is $\| u_j - z_j \|^2$, where $z_j$ is a second sequence used to add momentum (see Algorithm 2 in \cite{rioux2026discrete}). For the Variational (1) method the outer tolerance is always met first, but this is not always the case for Variational (L) and Accelerated variational (L) as the theoretical stepsizes are too conservative. %

For the Gaussian-to-ring example, the gradient ascent method does not meet the prescribed tolerance for solving the regularized convex subproblem for the first iterate, but does meet it first for the subsequent iterations. 
For all other examples this inner tolerance is always met before the maximum number of inner iterations is reached.

\paragraph{Aggregation and plotting.}
The stationarity curves show the geometric mean of the iterate differences plus one standard deviation in base-10 log space. Objective curves show the arithmetic mean of the \gls*{egw} values over runs and one
standard deviation. Running minima are computed within
each seed before aggregation. The aggregate curves are only plotted when at least half of the seeds converge to avoid misleading conclusions. %

 \Cref{tab:runtime} compiles the average runtime for each solver  excluding steps such as decomposing the matrix  for the variational methods. The large standard deviation for the Gaussian-ring example stems from the fact that  not all of the runs reach the prescribed tolerance in the chosen number of iterations.

\paragraph{Example from \cite{houry2026}}

For this example, we use the standard \texttt{sinkhorn-knopp} method with the stopping threshold set to $10^{-9}$ and at most
$10\,000$ iterations. All methods use a maximum of $J=300$ iterations and terminate early when $\|x_j-\tilde{x}_j\|_1^2<10^{-6}$ for \gls*{amd} and $\|u_j - u_{j+1}\|^2<10^{-6}$ for the dual gradient method. The regularized quadratic subproblems in the dual gradient method are solved using gradient ascent with a stepsize of $0.5$ and a maximum number of iterations set to  $100$.

The cost matrices are set to
\[
(\mathrm K_0^{(p)})_{ii'}=\|x_i-x_{i'}\|_2^p,\qquad
(\mathrm K_1^{(p)})_{kk'}=\|y_k-y_{k'}\|_2^p,
\qquad p\in\{2,6\},
\]
where the support points are described in \Cref{app:dataset-generation}. When $p=2$ the costs are CNT and are not CNT when $p=6$. We also consider two choices of regularization parameter, $\varepsilon=0.01$ and $\varepsilon=0.1$.

We visualize the couplings at different iterations for these  settings in \Cref{fig:couplings-c6,fig:couplings-c2}; these plots connect each point in the two datasets with a line whose width and opacity scale proportionally to the weight that the given coupling matrix assigns to that pair.

\section{Additional results}\label{app:additional-results}
\subsection{Gaussian-ring interaction objective}\label{app:interaction-objective}
When $p=2$, 
 \[
 \begin{aligned}
\mathsf{EGW}^{\varepsilon}_{2}(\mu_0,\mu_1)&=\int \kappa_0(x,x')^2 d\mu_0\otimes \mu_0(x,x') + \int \kappa_1(y,y')^2 d\mu_1\otimes \mu_1(y,y') 
\\
&+ \inf_{\pi\in\Pi(\mu_0,\mu_1)}\left\{-2\int \kappa_0(x,x')\kappa_1(y,y')d\pi\otimes \pi(x,y,x',y')+\varepsilon\mathsf{KL}(\pi\|\mu_0\otimes \mu_1)\right\},
\end{aligned} 
 \]
 so that
  the $\mathsf{EGW}$ objective decomposes into a constant term that depend only on the marginals and an interaction term. It suffices, therefore, to solve the minimization problem for the interaction term and hence only the interaction part is updated as the optimization progresses. 
  
  As explained in the main text, for the Gaussian-ring example these constant terms are excessively large as the cost is taken to be $\|\cdot\|^6$. Thus, in \Cref{fig:interaction-objective},  we plot the progress of the interaction term only as a function of the number of iterations for the various methods to get a clearer picture of the optimization progress. In particular, we note that the progress for the classical \gls*{md} method stalls and appears to oscillate.  %

\begin{figure}[!htb]
    \centering
    \includegraphics[width=.8\linewidth]{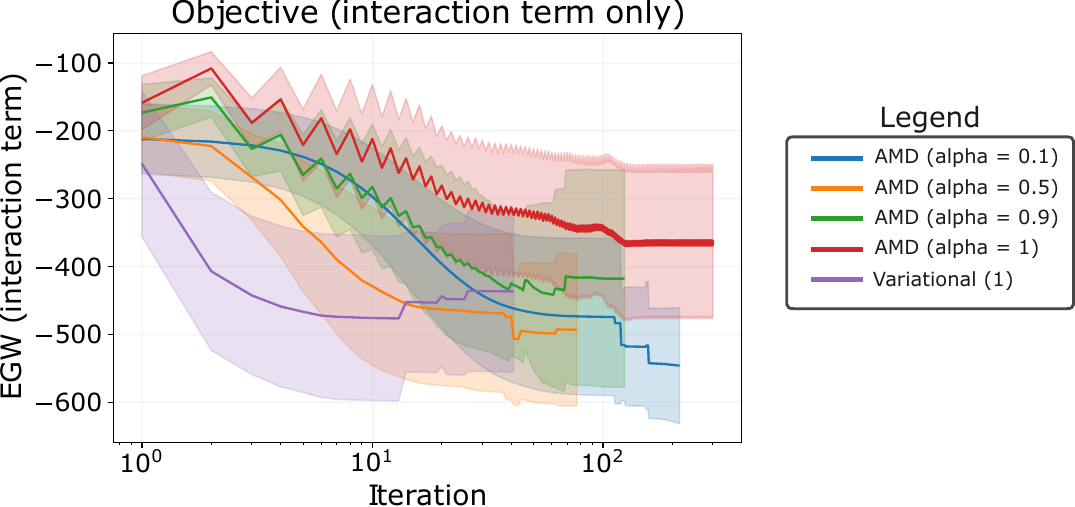}
    \caption{The interaction term from the EGW objective plotted for increasing number of iterations computed using the outputs of \gls*{amd} for $\alpha \in \{0.1, 0.5,0.9,1\}$ and the Variational (1) method.}
    \label{fig:interaction-objective}
\end{figure}

\subsection{Theoretical stepsizes}\label{app:theoretical-step}

As mentioned in the main text the theoretical stepsize from \Cref{prop:md-average-exact} or \Cref{prop:md-average-inexact} can be too conservative for \gls*{amd} and the method is seen to converge for larger step sizes as shown in \Cref{sec:experiments}. In this section, we report the theoretical stepsize choices for \gls*{amd} ($\alpha$) and the variational methods ($1/L$) for the Gaussian mixture and sphere examples. We highlight that we used the theoretical $\alpha$ from \Cref{prop:md-average-inexact}. In addition, we show in \Cref{fig:theoretical-amd} how \gls*{amd} with the theoretical stepsize compares with the other values of $\alpha$ considered in the text.
 Overall, the theoretical stepsize leads to slower convergence, %
 but still reaches approximate stationarity in the prescribed number of iterations in the sphere example. Experimental details are compiled in \Cref{app:experimental-settings}.

\begin{table}[!htb]
\centering
\small
\begin{tabular}{lccc}
\toprule
Method&Quantity & Gaussian mixture & Sphere \\
\midrule
AMD  &$\alpha$ & $(3.30\pm0.53)\times 10^{-4}$ & $(1.53\pm0.42)\times 10^{-3}$  \\
Variational ($L$) &$L_{\mathrm{var}}$ & $706 \pm 66$ & $15355 \pm 10740$ \\
Accelerated variational ($L$) &$L_{\mathrm{acc}}$ & $704 \pm 66$ & $15353 \pm 10740$ \\
\bottomrule
\end{tabular}
\caption{Theoretical step sizes, mean $\pm$ one standard deviation over 10 seeds.}
\label{tab:theoretical-steps}
\end{table}

\begin{figure}[!htb]
    \centering
    \includegraphics[width=.9\linewidth]{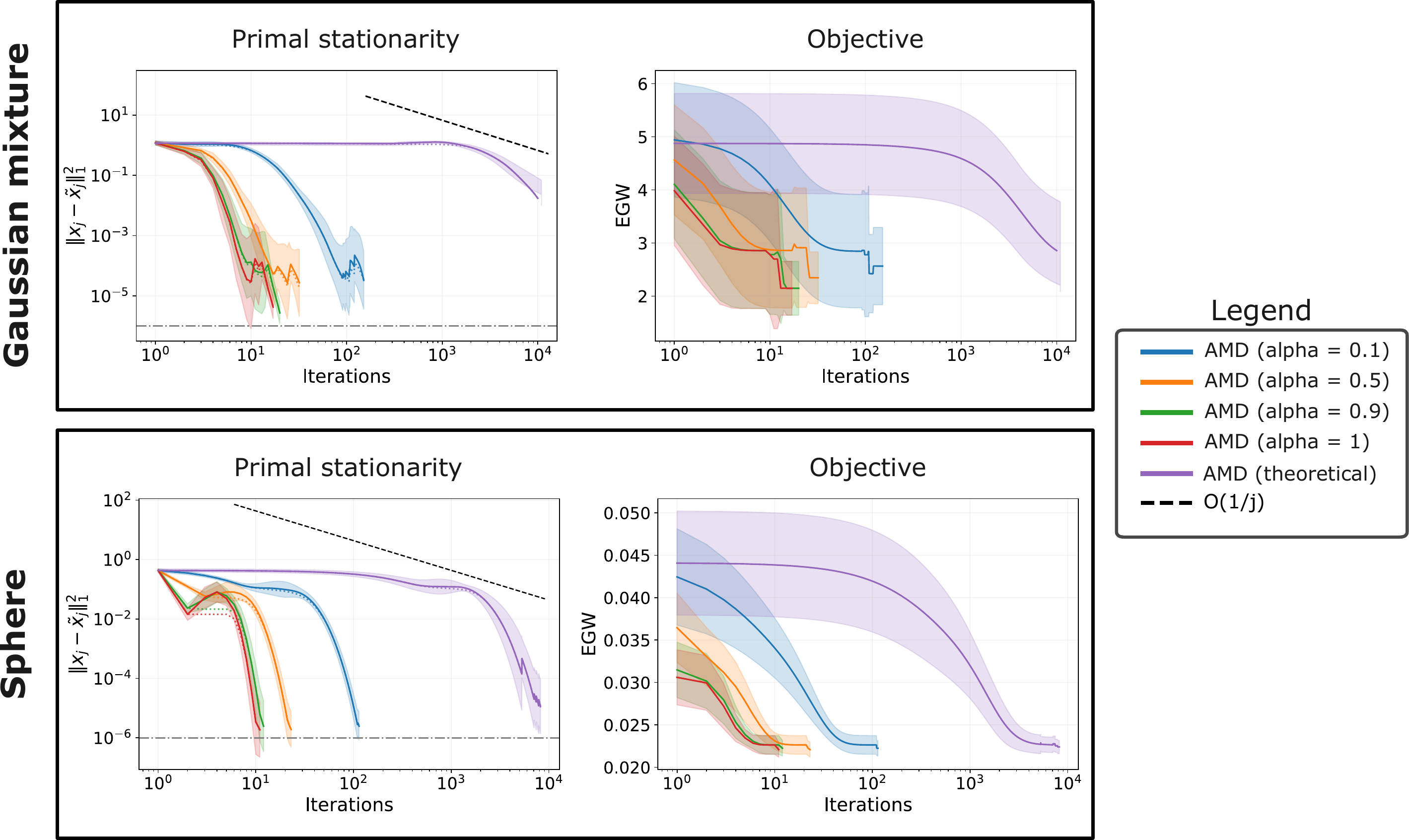}
    \caption{The primal stationarity and objective plots for \gls*{amd} for $\alpha \in \{0.1, 0.5,0.9,1, \text{theoretical}\}$ for the Gaussian mixture and sphere examples. The theoretical step sizes per experiment are reported in \Cref{tab:theoretical-steps}.}
    \label{fig:theoretical-amd}
\end{figure}

\subsection{Coupling visualizations}

Below we visualize the coupling at the last few iterations for the example adapted from \cite{houry2026}. We recall that the lines between the datasets are plotted such that their thickness and opacity is equal to the mass that the coupling assigns to that particular pair. 

When the intra-domain cost is $\|\cdot\|^6$ which is not CNT and $\varepsilon = 0.01$, the couplings for AMD with $\alpha =0.9$ and $\alpha =1$ are seen to alternate between two configurations, indicating that the method has failed to converge. When $\varepsilon=0.1$, this issue does not occur, indicating that the failure to converge is also caused by the choice of regularization parameter. When the intra-domain cost is set to  $\|\cdot\|^2$ which is CNT, all methods converge as expected.

\begin{figure}[!htb]
    \centering
    \includegraphics[width=.8\linewidth]{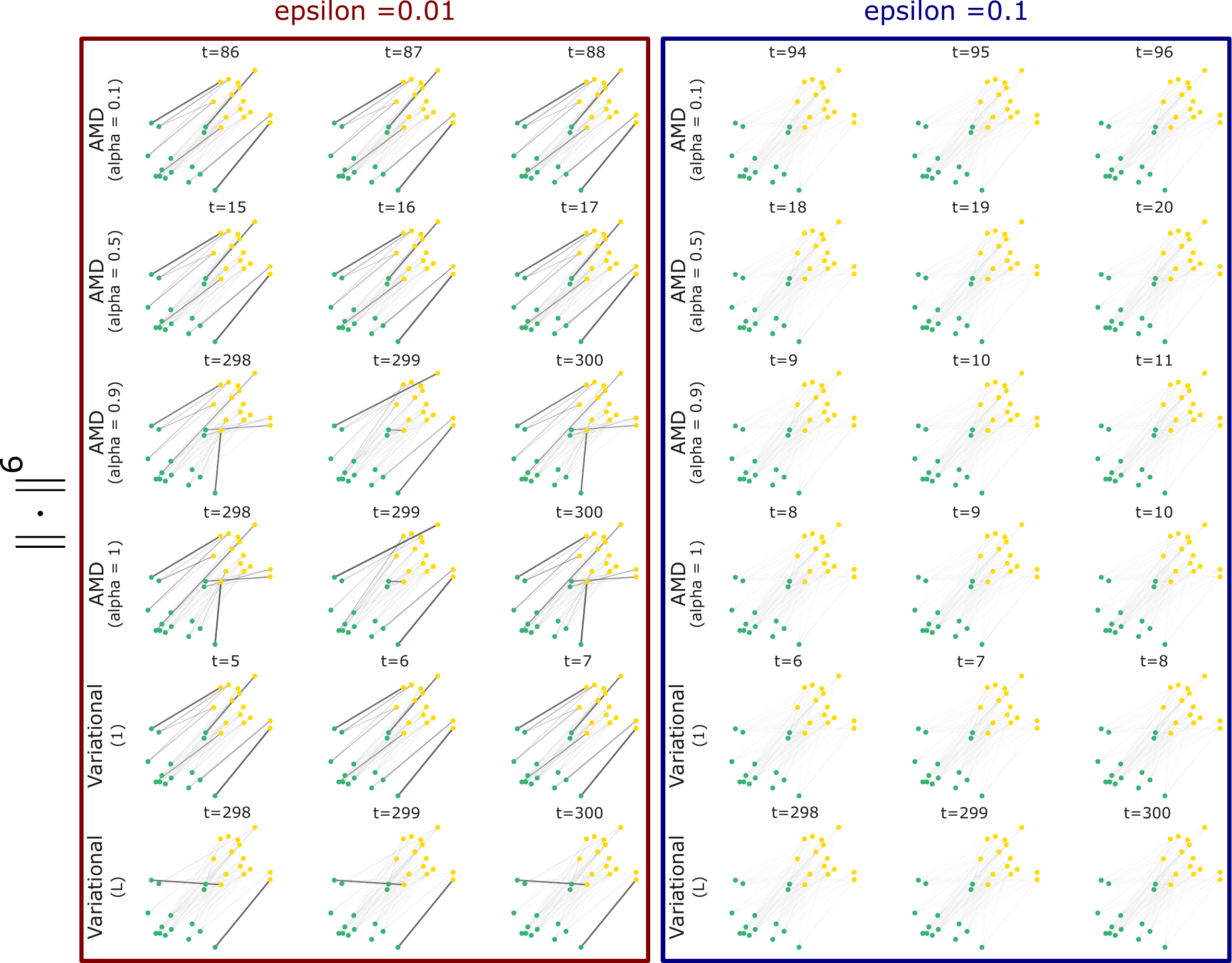}
    \caption{   
    Couplings at the final iterations for different algorithms applied to the example adapted from \cite{houry2026} when the intra-domain cost is set to $\|\cdot\|^6$. The support points of $\mu_0$ are in green and those from $\mu_1$ are in yellow. On the left, the regularization is set to $\varepsilon = 0.01$ and $\varepsilon=0.1$ on the right.} 
    \label{fig:couplings-c6}
\end{figure}

\newpage
\begin{figure}[!htb]
    \centering
    \includegraphics[width=.8\linewidth]{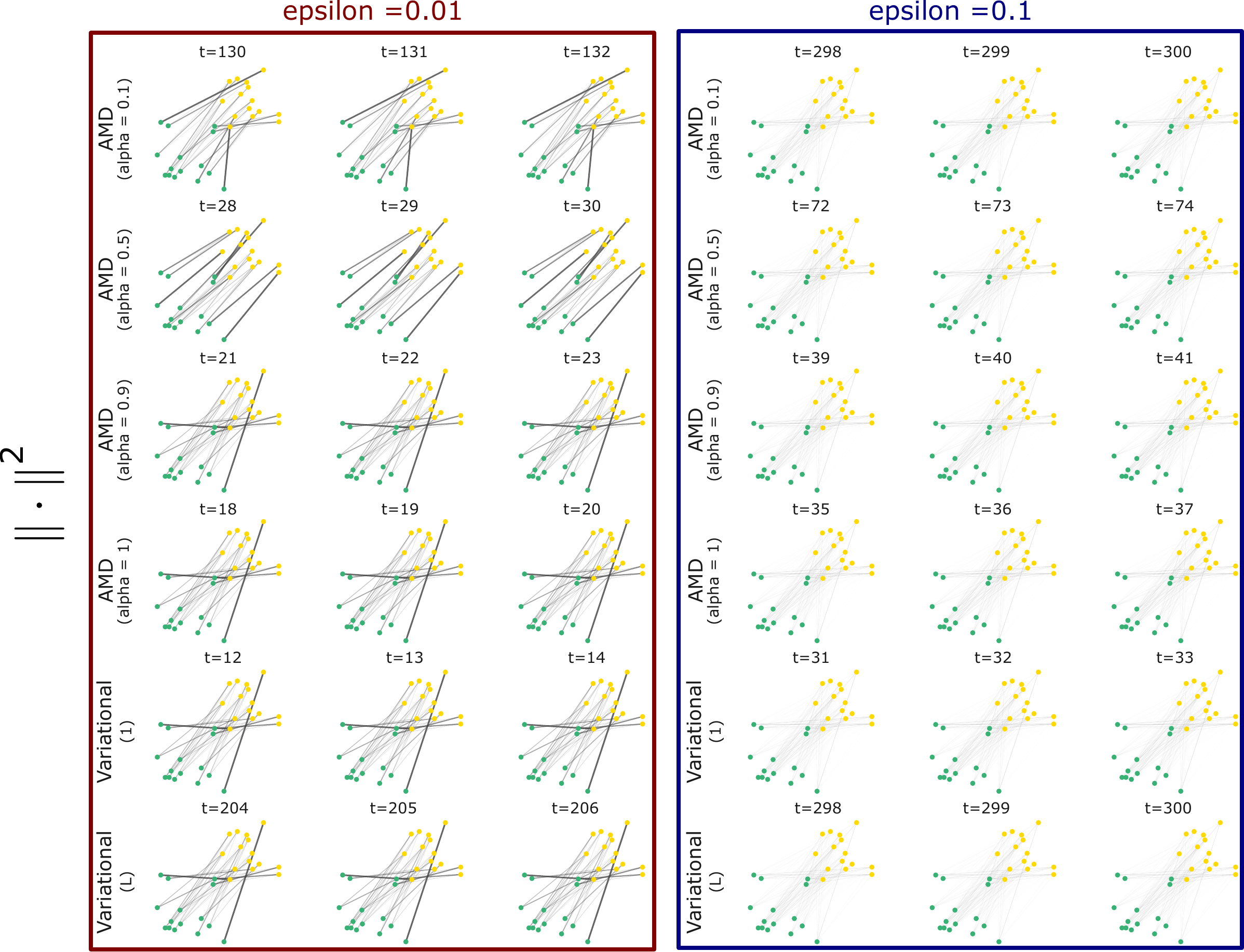}
    \caption{Couplings at the final iterations for different algorithms applied to the example adapted from \cite{houry2026} when the intra-domain cost is set to $\|\cdot\|^2$. The support points of $\mu_0$ are in green and those from $\mu_1$ are in yellow. On the left, the regularization is set to $\varepsilon = 0.01$ and $\varepsilon=0.1$ on the right}
    \label{fig:couplings-c2}
\end{figure}

\subsection{Complexity in terms of average Sinkhorn calls}
\Cref{tab:sinkhorn-calls} reports the average number of calls to the Sinkhorn solver for each algorithm together with one standard deviation. %

\begin{table}[!htb]
\centering
\small
\begin{tabular}{lccc}
\toprule
Method & Gaussian mixture & Sphere & Gaussian-ring \\
\midrule
AMD ($\alpha=0.1$)                 & $149 \pm 49$          & $118 \pm 9$          & $200 \pm 72$ \\
AMD ($\alpha=0.5$)                 & $30.8 \pm 10.0$       & $23.4 \pm 2.1$       & $99.1 \pm 77.6$ \\
AMD ($\alpha=0.9$)                 & $17.0 \pm 5.4$        & $12.5 \pm 1.4$       & $158 \pm 118$ \\
AMD ($\alpha=1$)                   & $15.0 \pm 5.0$        & $11.1 \pm 1.1$       & $300 \pm 0$ \\
Variational (1)           & $68.6 \pm 21.0$       & $32.2 \pm 2.4$       & $2458 \pm 1350$ \\
Variational ($L$)           & $14517 \pm 3407$      & $20006 \pm 0.5$      & --- \\
Accelerated variational ($L$) & $6938 \pm 1544$     & $15913 \pm 4641$     & --- \\
\bottomrule
\end{tabular}
\caption{Mean number of Sinkhorn calls ($\pm$ one standard deviation) over 10 seeds.}
\label{tab:sinkhorn-calls}
\end{table}

\end{document}